\documentclass[twoside,11pt]{article}

\usepackage{moreverb,url}
\usepackage{nameref}
\usepackage{times}
\usepackage{soul}
\usepackage{multicol}
\usepackage[table,xcdraw,svgnames]{xcolor}
\usepackage{amsmath}
\usepackage{algorithm}
\usepackage{subcaption}
\usepackage{algorithmic}
\usepackage{eqparbox}

\usepackage{amsfonts}
\usepackage{multirow}

\usepackage{enumitem}
\usepackage{bbm}
\usepackage{booktabs}
\usepackage{tikz}
\usepackage{amsfonts}
\usepackage{multirow}
\usepackage{mathtools}

\usepackage{siunitx}
\usepackage{float}
\usepackage{lastpage}

\usepackage{jmlr2e}

\usepackage[figuresright]{rotating}
\usepackage{float}

\let\truehypersetup\hypersetup
\renewcommand\hypersetup[1]{}
\usepackage{bigfoot}
\let\hypersetup\truehypersetup
\hypersetup{ hidelinks }

\DeclareNewFootnote{AAffil}[arabic]
\DeclareNewFootnote{ANote}[fnsymbol]

\usepackage{etoolbox}
\makeatletter
\patchcmd\maketitle{\def\@makefnmark{\rlap{\@textsuperscript{\normalfont\@thefnmark}}}}{}{}{}
\makeatother

\makeatletter
\def\thanksAAffil#1{
  \footnotemarkAAffil\protected@xdef\@thanks{\@thanks%
        \protect\footnotetextAAffil[\the \c@footnoteAAffil]{#1}}%
}
\def\thanksANote#1{%
  \footnotemarkANote%
  \protected@xdef\@thanks{\@thanks%
        \protect\footnotetextANote[\the \c@footnoteANote]{#1}}%
}
\makeatother

\DeclareMathOperator*{\argmax}{arg\,max}

\newcommand\BibTeX{{\rmfamily B\kern-.05em \textsc{i\kern-.025em b}\kern-.08em
T\kern-.1667em\lower.7ex\hbox{E}\kern-.125emX}}
\definecolor{start_col}{HTML}{ff7f00}
\definecolor{end_col}{HTML}{fff4ea}
\definecolor{mid_col}{HTML}{ffb977}

\newcommand\Mark[2][8.4]{%
  \rlap{\tikz[]{
        \shade[left color=start_col, right color=end_col, middle color=mid_col]
               (0,0) rectangle ++(#1*#2/100,0.3);}%
  }%
}

\jmlrheading{1}{2023}{1-\pageref{LastPage}}{4/00}{10/00}{Paleja$^*$, Chen$^*$, Niu$^*$, Silva, Li, Zhang, Ritchie, Choi, Chang, Tseng, Wang, Nageshrao, Gombolay}

\ShortHeadings{Learning Interpretable Continuous Control Policies}{Paleja et al. }
\firstpageno{1}

\begin{document}

\title{Interpretable Reinforcement Learning for Robotics and Continuous Control}

\author{\name Rohan Paleja%
        \thanksANote{Equal contribution}$\ ^{,}$\thanksAAffil{Georgia Institute of Technology,
       Atlanta, GA 30332, USA}$\ ^{,}$\thanksANote{Corresponding author} \email rpaleja3@gatech.edu
       \AND
       \name Letian Chen\footnotemarkANote[1]$\ ^{,}$\footnotemarkAAffil[1] \email letian.chen@gatech.edu
       \AND
       \name Yaru Niu%
       \footnotemarkANote[1]$\ ^{,}$\thanksAAffil{Carnegie Mellon University, Pittsburgh, PA 15213, USA} \email yarun@andrew.cmu.edu
       \AND
       \name Andrew Silva\footnotemarkAAffil[1] \email asilva9@gatech.edu
       \AND
       \name Zhaoxin Li\footnotemarkAAffil[1] \email zli3088@gatech.edu
       \AND
        \name Songan Zhang\thanksAAffil{Ford Motor Company, Dearborn, MI 48120, USA} \email szhan117@ford.com
              \AND
       \name Chace Ritchie\footnotemarkAAffil[1] \email chaceritchiedev@gmail.com
       \AND
       \name Sugju Choi\footnotemarkAAffil[1] \email schoi24@gatech.edu
        \AND
       \name Kimberlee Chestnut Chang\thanksAAffil{MIT Lincoln Laboratory, Lexington, MA 02421, USA} \email chestnut@ll.mit.edu
        \AND
        \name Hongtei Eric Tseng\footnotemarkAAffil[3] \email hongtei.tseng@gmail.com
        \AND
        \name Yan Wang\footnotemarkAAffil[3] \email ywang21@ford.com 
        \AND
        \name Subramanya Nageshrao\footnotemarkAAffil[3] \email  snageshr@ford.com
       \AND
       \name Matthew Gombolay\footnotemarkAAffil[1] \email matthew.gombolay@cc.gatech.edu}
\editor{}

\maketitle

\vspace{-15mm} 

\begin{abstract}
Interpretability in machine learning is critical for the safe deployment of learned policies across legally-regulated and safety-critical domains. While gradient-based approaches in reinforcement learning have achieved tremendous success in learning policies for continuous control problems such as robotics and autonomous driving, the lack of interpretability is a fundamental barrier to adoption. We propose Interpretable Continuous Control Trees (ICCTs), a tree-based model that can be optimized via modern, gradient-based, reinforcement learning approaches to produce high-performing, interpretable policies.
The key to our approach is a procedure for allowing direct optimization in a sparse decision-tree-like representation.
We validate ICCTs against baselines across six domains, showing that ICCTs are capable of learning policies that parity or outperform baselines by up to 33$\%$ in autonomous driving scenarios while achieving a $300$x-$600$x reduction in the number of parameters against deep learning baselines. We prove that ICCTs can serve as universal function approximators and display analytically that ICCTs can be verified in linear time. Furthermore, we deploy ICCTs in two realistic driving domains, based on interstate Highway-94 and 280 in the US. Finally, we verify ICCT's utility with end-users and find that ICCTs are rated easier to simulate, quicker to validate, and more interpretable than neural networks.
\end{abstract}


\section{Introduction}
Reinforcement learning (RL) with deep function approximators has enabled the generation of high-performance continuous control policies across a variety of complex domains, from robotics \cite{Lillicrap2016ContinuousCW} and autonomous driving \cite{Wu2017FlowAA} to protein folding \cite{Jumper2021HighlyAP} and traffic regulation \cite{Cui2021ScalableMD}. These approaches hold tremendous promise in real-world applicability and have the potential
to increase traffic safety \citep{Katrakazas2015RealtimeMP}, decrease traffic congestion, increase average traffic speed in human-driven traffic \citep{Cui2021ScalableMD}, reduce $CO_2$ emissions, and allow for more affordable transportation \citep{Abe2019IntroducingAB}. 
However, while the performance of these controllers opens up the possibility of real-world adoption, the conventional deep-RL policies used in prior work \citep{Lillicrap2016ContinuousCW,Wu2017FlowAA,Cui2021ScalableMD} lack \emph{interpretability}, limiting deployability in safety-critical and legally-regulated domains \citep{doshi2017towards,letham2015interpretable,bhatt2019explainable,voigt2017eu}. 

White-box approaches, as opposed to typical black-box models (e.g., deep neural networks) used in deep-RL, model decision processes in a human-readable representation. Such approaches afford interpretability, allowing users to gain insight into the model's decision-making behavior. \textcolor{black}{In autonomous driving, such models would provide insurance companies, law enforcement, developers, and passengers with insight into how an autonomous vehicle (AV) reasons about state features and makes decisions.} Utilizing such white-box approaches within machine learning is necessary for the deployment of autonomous vehicles and essential in building trust, ensuring safety, and enabling developers to inspect and verify policies before deploying them to the real world \citep{olah2018building,hendricks2018generating,anne2018grounding}. In this work, we present a novel tree-based architecture that affords gradient-based optimization with modern RL techniques to produce high-performance, interpretable policies \textcolor{black}{for continuous control applications. We note that our proposed architecture can be applied to a multitude} of continuous control problems ranging from robotics \citep{Lillicrap2016ContinuousCW}, protein folding \citep{Jumper2021HighlyAP}, and traffic regulation \citep{Cui2021ScalableMD} to high-speed autonomous driving~\citep{Wu2017FlowAA} and autopilots for landing spacecraft~\citep{banerjee2018nonlinear}. 

Prior work~\citep{olah2018building,Kim2015InteractiveAI,hendricks2018generating} has attempted to approximate interpretability via explainability, a practice that can have severe consequences \citep{Rudin2018StopEB}. 
While the explanations produced in prior work can help to partially explain the behavior of a control policy, the explanations are not guaranteed to be accurate or generally applicable across the state-space, leading to erroneous conclusions and a lack of accountability of predictive models \citep{Rudin2018StopEB}. \textcolor{black}{In autonomous driving, where understanding a decision-model is critical to avoiding collisions, local explanations are insufficient.} An \textit{interpretable model}, instead, provides a transparent \textit{global} representation of a policy's behavior. This model can be understood directly by its structure and parameters \citep{Ciravegna2021LogicEN} (e.g., linear models, decision trees, and our ICCTs), \textcolor{black}{offering verifiability} and guarantees that are not afforded by post-hoc explainability frameworks. Few works have attempted to learn an interpretable model directly; rather, prior work has attempted policy distillation to a decision tree \citep{Frosst2017DistillingAN,viper,wu2018beyond} or imitation learning via a decision tree across trajectories generated via a deep model \citep{Bastani2018VerifiableRL}, leaving much to be desired. Interpretable RL remains an open challenge \citep{Rudin2021InterpretableML}. In this work, we directly produce high-performance, interpretable policies represented by a minimalistic tree-based architecture augmented with low-fidelity linear controllers via RL, \textcolor{black}{providing a novel interpretable RL architecture}. Our Interpretable Continuous Control Trees are human-readable, \textcolor{black}{allow for closed-form verification (associated with safety guarantees)}, and parity or outperform baselines by up to $33\%$ in autonomous driving scenarios. In this work, our key contributions are:

\begin{enumerate}[leftmargin=*]
    \item We propose Interpretable Continuous Control Trees (ICCTs), a novel tree-based architecture that can be optimized via gradient descent with modern RL algorithms to produce high-performance, interpretable continuous control policies. We provide several extensions to prior differentiable decision tree (DDT) frameworks to increase expressivity and allow for direct optimization of a sparse decision-tree-like representation.
    \item We show that our ICCTs are universal function approximators and can thus be utilized to learn continuous control policies in any domain, assuming that the ICCT has a reasonable depth.
    \item We empirically validate ICCTs across six continuous control domains, including four autonomous driving scenarios. Further, we demonstrate ICCT's ability to learn driving policies in complex domains grounded in realistic real-world lane geometries, including the I-94 highway in Michigan, USA, and the I-280 highway in California, USA. 
    \item We show that our ICCTs can be verified in linear time, a vital characteristic in assessing and understanding a model's behavior under a set of inputs. Whereas black-box approaches are difficult to verify, our ICCTs can be verified quickly, providing the possibility of safety guarantees and opening the door for safe, real-world adoption.
    \item We demonstrate the utility of our ICCTs with end-users through a human-subjects study (N=34) and show that the ICCT is rated by users as easier to simulate, quicker to validate, and more interpretable than neural networks. 
\end{enumerate}

This paper presents our work in the field of Interpretable Reinforcement Learning and Explainable AI, highlighting prior techniques in Section 2. In Section 3, we introduce the necessary preliminary work on Differentiable Decision Trees and Reinforcement Learning. Our Methodology, covered in Section 4, outlines the ICCT architecture and the Differentiable Crispification technique used for enabling policy updates via gradient-based techniques. Section 5 establishes ICCTs as universal function approximators, and Section 6 analyzes the time complexity for model verification. In Sections 7 and 8, we introduce and evaluate our ICCT across six continuous control domains. In Section 9, we offer a qualitative example of a learned policy within the Lunar Lander domain to showcase the interpretability of our model. Section 10 explores the tradeoff between performance, leaf controller sparsity, and tree depth. We compare our differential crispification method to the Gumbel-Softmax procedure in Section 11. Section 12 demonstrates the interpretability and utility of ICCTs through a 14-car physical robot demonstration. In Section 13, we introduce two realistic driving domains based on real-world lane geometries and find that trained ICCTs can produce safe, high-performance behavior that follows traffic regulations. Finally, Section 14 presents a user study evaluating the interpretability of our model.

\section{Related Work}
\label{sec:related_work}
\textcolor{black}{Due to recent accidents with autonomous vehicles (cf.~\citet{Yurtsever2020ASO}), there has been growing interest in developing Explainable AI (xAI) approaches to understand a control policy's decision-making and ensure robust and safe operation.} xAI is concerned with understanding and interpreting the behavior of AI systems \citep{Linardatos2021ExplainableAA}. In recent years, the necessity for human-understandable models has increased greatly for safety-critical and legally-regulated domains, many of which involve continuous control (e.g., specifying joint torques for a robot arm or the steering angle for an autonomous vehicle) \citep{Kim2017InterpretableLF,doshi2017towards}. 
In such domains, prior work \citep{Schulman2017ProximalPO,Lillicrap2016ContinuousCW,Fujimoto2018AddressingFA,haarnoja2018soft} has typically used highly-parameterized deep neural networks in order to learn high-performance policies, completely lacking in model transparency.

Interpretable machine learning approaches refers to a subset of xAI techniques that produce globally transparent policies (i.e., humans can inspect the entire model, as in a decision tree \citep{Breiman1983ClassificationAR,basak2004online,olaru2003complete} or rule list \citep{angelino2017learning,weiss1995rule,letham2015interpretable,chen2017optimization}). 
Decision trees \citep{Breiman1983ClassificationAR} represent a hierarchical structure where an input decision can be traced to an output via evaluation of decision nodes (i.e., ``test" on an attribute) until arrival at a leaf node. 
Decision nodes within the tree are able to split the problem space into meaningful subspaces, simplifying the problem as the tree gets deeper \citep{laptev2014convolutional,kontschieder2015deep,tanno2018adaptive}.
Decision trees provide \textit{global} explanations of a decision-making policy that are valid throughout the input space \citep{Barbiero2021PyTorchEA}, as opposed to local explanations typically provided via ``post-hoc"  explainability techniques \citep{Ribeiro2019ModelAgnosticEA,Silva2021EncodingHD,Paleja2020InterpretableAP}. 
Several approaches have attempted to distill trained black-box neural network models into decision trees \citep{Wu2020RegionalTR,viper}. \textcolor{black}{While these approaches produce interpretable models, the resulting model is an approximation of the neural network rather than a true representation of the underlying model. Our work, instead, directly learns an interpretable tree-based policy via reinforcement learning, producing a model that can be directly verified without utilizing error-prone post-hoc explainability techniques.
We emphasize that explainability stands in contrast to interpretability}, as explanations may fail to capture the true decision-making process of a model or may apply only to a local region of the decision-space, thereby preventing a human from building a clear or accurate mental model of the entire policy ~\citep{Rudin2018StopEB,lipton2018mythos,adadi2018peeking, Paleja_Utility_of_xAI}. 

Recently, \citet{Rudin2021InterpretableML} presented a set of grand challenges in interpretable machine learning to guide the field toward solving critical research problems that must be solved before machine learning can be safely deployed within the real world. 
In this work, we present a solution to directly assess two challenges: (1) Optimizing sparse logical models such as decision trees and (10) Interpretable reinforcement learning. We propose a novel high-performing, sparse tree-based architecture, Interpretable Continuous Control Trees (ICCTs), which allows end-users to directly inspect the decision-making policy and developers to verify the policy for safety guarantees. 




\section{Preliminaries}
In this section, we review differentiable decision trees (DDTs)
and reinforcement learning. 
\begin{figure*}[t]
    \centering
    \includegraphics[width=\textwidth]{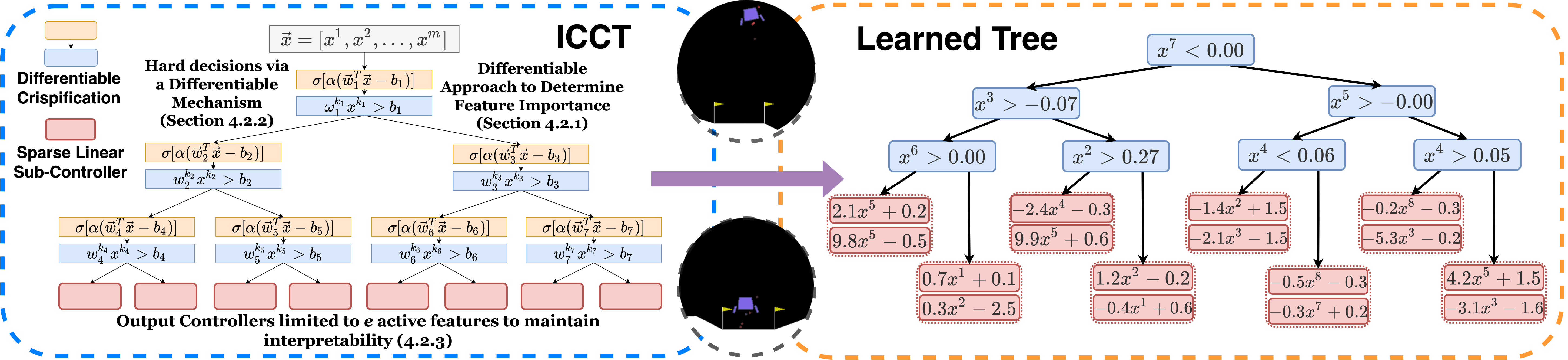}
    \caption{\textcolor{black}{The ICCT framework (left) displays decision nodes, both in their fuzzy form (orange blocks) and crisp form (blue blocks\protect\footnotemark), and sparse linear leaf controllers with pointers to sections discussing our contributions. A learned representation of a high-performing ICCT policy in Lunar Lander (right) displays the interpretability of our ICCTs. Each decision node is conditioned upon only a single feature and the sparse linear controllers (to control the main engine throttle and left/right thrusters) are set to have only \textbf{one} active feature.}}
    
    \label{fig:dt}
\end{figure*}

\subsection{Differentiable Decision Trees (DDTs)}
\label{sec:ddts}
\textcolor{black}{Prior work has proposed differentiable decision trees (DDTs)  \citep{suarez1999globally,Silva2021EncodingHD,Paleja2020InterpretableAP, tambwekar_ral} -- a neural network architecture that takes the topology of a decision tree (DT). Similar to a decision tree, DDTs contain decision nodes and leaf nodes; however, each decision node within the DDT utilizes a sigmoid activation function (i.e., a ``soft" decision) instead of a Boolean decision (i.e., a ``hard" decision). Each decision node, $i$, is represented by a sigmoid function, displayed in Equation \ref{eq:ddt_decision_node}.}
\begin{equation}
\label{eq:ddt_decision_node}
y_i = \frac{1}{1+\exp(-\alpha(\vec{w}_{i}^{T} \vec{x} - b_i))}
\end{equation}
\textcolor{black}{Here, the features vector describing the current state, $\vec{x}$, are weighted by $\vec{w}_i$, and a splitting criterion, $b_i$, is subtracted to form the splitting rule. $y_i$ is the probability of decision node $i$ evaluating to \textsc{True}, and $\alpha$ governs the steepness of the sigmoid activation, where $\alpha \rightarrow \infty$ results in a step function.} Prior work with discrete-action DDTs modeled each leaf node with a probability distribution over possible output classes \citep{Silva2021EncodingHD, Paleja2020InterpretableAP}. Leaf node distributions, $\vec{L}$, are then weighted by the probability of reaching the respective leaf and summed to produce a final action distribution over possible outputs.\footnotetext{For figure simplicity, when displaying the crisp node (blue block), we assume $\alpha>0$ in the fuzzy node (orange block). If $\alpha<0$, the sign of the inequality would be flipped (i.e., $w_i^{k_i} x^{k_i}<b$).} \textcolor{black}{For a simple 4-leaf tree, this results in an \textit{fuzzy} output distribution with a complex interplay of node and leaf probabilities, displayed in Equation \ref{eq:fuzzy_output}.}
\begin{equation}
\label{eq:fuzzy_output}
    P(a|x) = \vec{L}_1 (y_1 * y_2) + \vec{L}_2 (y_1 * (1-y_1)) + \vec{L}_3 ((1-y_1)*y_3) + \vec{L}_4 ((1-y_1)(1-y_3))
\end{equation}

\subsubsection{Conversion of a DDT to a DT}
\label{sec:prolo_crisp}
DDTs with decision nodes represented in the form of Equation \ref{eq:ddt_decision_node} are not interpretable.
As DDTs maintain a one-to-one correspondence to DTs with respect to their structure, prior work \citep{Silva2021EncodingHD, Paleja2020InterpretableAP} proposed a methodology to convert a DDT into an interpretable decision tree (a process termed ``crispification").
To create an interpretable, ``crisp" tree from a differentiable form of the tree, prior work adopted a simplistic procedure. Starting with the differentiable form, prior work first converts each decision node from a linear combination of all variables into a single feature check (i.e., a 2-arity predicate with a variable and a threshold). The feature reduction is accomplished by considering the feature dimension corresponding to the weight with the largest magnitude \textcolor{black}{(i.e., most impactful)}, $ k = \argmax_{j}|w_{i}^{j}|$, where $j$ represents the feature dimension,
resulting in the decision node representation $y_i = \sigma(\alpha(w^k_ix^k-b_i))$.
The sigmoid steepness, $\alpha$, is also set to infinity, resulting in a ``hard" decision (branch left OR right) \citep{Silva2021EncodingHD,Paleja2020InterpretableAP}. After applying this procedure to each decision node, decision nodes are represented by  $y_i = \mathbbm{1}(w^k_ix^k-b_i > 0)$.
As each leaf node is represented as a probability mass function over output classes in prior work, each leaf node, $l_d$, indexed by $d$, must be modified to produce a single output class, $o_d$, during crispification. As such, we can apply an argument max, $o_d = \argmax_a l^a_d$, where $a$ denotes the action dimension, to find the maximum valued class within the $d$-th leaf distribution.

\noindent\textit{Deficiency of direct conversion from DDT to DT:} This simplistic crispification procedure results in an interpretable crisp tree that is inconsistent with the original DDT (model differences arise from each \textit{argmax} operation and setting the signmoid steepness, $\alpha$, to infinity). These inconsistencies can lead to performance degradation of the interpretable model, as we show in Section \ref{sec:results}, and results in an interpretable model that is not representative of and inconsistent with the model learned via reinforcement learning. 

\textcolor{black}{In our work, we address these limitations by designing a novel architecture that updates its parameters via gradient descent while maintaining an interpretable decision-tree-like representation, thereby avoiding any inconsistencies generated through a post-hoc crispification procedure.
To the best of our knowledge, we are the first work to deploy an interpretable tree-based framework for continuous control.}

\subsection{Reinforcement Learning (RL)}
A Markov Decision Process (MDP) $M$ is defined as a 6-tuple $\langle S,A,R,T,\gamma,\rho_0\rangle$. $S$ is the state-space, $A$ is the action-space, $R(s,a)$ is the reward received by an agent for executing action, $a$, in state, $s$, $T(s^\prime|s,a)$ is the probability of transitioning from state, $s$, to state, $s'$, when applying action, $a$, $\gamma\in [0,1]$ is the discount factor, and $\rho_0(s)$ is the initial state distribution. 
A policy, $\pi(a|s)$, gives the probability of an agent taking action, $a$, in state, $s$. 
The goal of RL is to find the optimal policy, 
$\pi^*=\arg\max_\pi \mathbb{E}_{\tau\sim\pi}\left[\sum_{t=0}^T{\gamma^tR(s_t,a_t)}\right]$ to maximize cumulative discounted reward, where $\tau=\langle s_0,a_0,\cdots,s_T,a_T\rangle$ is an agent's trajectory. 

In this work, while ICCTs are framework-agnotistic (i.e., ICCTs will work with any RL update rule), we proceed with Soft Actor-Critic (SAC) \citep{haarnoja2018soft} as our RL algorithm due to its learning stability and sample efficiency. The actor objective within SAC is given in Equation \ref{eqn:sac-actor}, where $Q_w(s_t,a_t)$ is expected, future discounted reward parameterized by $\omega$ and $\alpha_{\tau}$ is a temperature parameter that determines the relative importance of the stochastic policy entropy versus the reward.
\begin{align}
    \label{eqn:sac-actor}
        J_\pi(\theta) = \mathbb{E}_{s_t \sim\mathcal{D}}[\mathbb{E}_{a_t \sim\pi_\theta}[&\alpha_{\tau} \log(\pi_\theta(a_t|s_t)) -  Q_\omega(s_t, a_t)]]
\end{align}

\section{Method}
\label{sec:method}
In this section, we introduce our ICCTs, a novel interpretable reinforcement learning architecture. ICCTs are able to maintain interpretability while representing high-performance continuous control policies, making them suitable for applications that require trust and accountability, such as robotic manipulation and autonomous vehicle control. We provide several extensions to prior DDT frameworks within our proposed architecture, including 1) a differentiable crispification procedure allowing for optimization in a sparse decision-tree-like representation and 2) the addition of sparse linear leaf controllers to increase expressivity while maintaining legibility. 

\subsection{ICCT Architecture}
\label{sec:icct_arch}
Our ICCTs are initialized to be a symmetric decision tree with $N_l$ decision leaves and $N_l-1$ decision nodes. A depiction of our ICCT can be seen in Figure~\ref{fig:dt}, with decision leaves shown in red and decision nodes shown in blue. The tree depth can be determined by $\log_2(N_l)$. Each decision leaf is represented by a sparse linear controller that operates on $\vec{x}$. Decisions are routed via decision nodes toward a leaf controller, which is then used to produce the continuous control output (e.g., acceleration or steering wheel angle).
Our ICCT is similar to hierarchical models, which encompass a high-level controller that governs and coordinates multiple low-level controllers. Prior work has shown this to be a successful paradigm in continuous control \citep{Nachum2018DataEfficientHR}.

Each decision node, $i$, has an activation steepness weight, \textcolor{black}{$\alpha$}, associated weights, $\vec{w}_i$, with cardinality, $m$, matching that of the input feature vector, $\vec{x}$, and a scalar bias term, $b_i$, similar to that of Equation \ref{eq:ddt_decision_node}. 
\textcolor{black}{Each leaf node, $l_d$, where $d \in \{1,\dots,N_l\}$, contains per-leaf weights, $\vec{\beta}_d \in \mathbb{R}^{m}$, per-leaf selector weights that learn the relative importance of candidate features, $\vec{\theta}_d \in \mathbb{R}^{m}$}, per-leaf bias terms, $\vec{\phi}_{d} \in \mathbb{R}^{m}$,
and per-leaf scalar standard deviations, $\gamma_d$. 
We note that if the action space is multi-dimensional, then \emph{only} the leaf controllers (and associated weights) are expanded across $|A|$ dimensions, where $|A|$ is the cardinality of the action space. For each action dimension, the mean of the output action distribution is represented by the linear controller, $l_d$.
\begin{equation}
    \label{eq:complete_leaf}
    l_d \triangleq (\vec{u}\circ\vec{\beta}_d)^T (\vec{u}\circ\vec{x}) - \vec{u}^T\vec{\phi}_d
\end{equation}
Before enforcing leaf controller sparsity (i.e., by forcing the controller to condition upon only a subset of features, Section~\ref{sec:sparse_sub_controllers}), $\vec{u}=[1, ..., N_l]^T$ is an all-ones vector, representing the set of active features within the leaf node, in which case Equation~\ref{eq:complete_leaf} can be simplified as $l_d=\vec{\beta}_d^T \vec{x} - \vec{u}^T\vec{\phi}_d$. The output action can be determined via sampling ($a\sim \mathcal{N}(\vec{\beta}_d^T \vec{x} - \vec{u}^T\vec{\phi}_d, \gamma_d)$) during training and directly via the mean during runtime. We term decision nodes that are represented as Equation \ref{eq:ddt_decision_node} as fuzzy decision nodes, displayed by the orange rectangles within the left-hand side of Figure \ref{fig:dt}. Similarly, we term the leaf node, $l_d$, when it is represented in the dense representation of $\vec{\beta}_d^T \vec{x} - \vec{u}^T\vec{\phi}_d$, as a fuzzy leaf node. Here, we parameterize the bias term as a vector, $\vec{\phi}_d$, instead of a scalar as in our decision nodes to provide a corresponding bias for each feature and facilitate feature-wise optimization across different dimensions of the bias term.

Utilizing a novel differentiable crispification procedure to convert fuzzy decision nodes into crisp decision nodes (i.e., 2-arity predicate with a variable and a threshold) and fuzzy leaf nodes into sparse leaf nodes (i.e., linear controller conditioned upon a small subset of features), our model representation follows that of a decision tree with sparse linear controllers at the leafs (shown on the right-hand side of Figure \ref{fig:dt}). We further discuss our differentiable crispification procedure in Sections \ref{sec:decision_nodes_crispification}-\ref{sec:decision_outcome_crispification} (i.e., the mechanism that translates orange blocks to blue within Figure \ref{fig:dt}) and leaf controller sparsification procedure in Section \ref{sec:sparse_sub_controllers}. 

While decision trees (DT) are generally considered interpretable \citep{letham2015interpretable}, trees of arbitrarily large depths can be difficult to understand \citep{Ghose2020InterpretabilityWA} and simulate \citep{lipton2018mythos}. A sufficiently sparse DT is desirable and considered interpretable \citep{lakkaraju}. Furthermore, the utilization of linear controllers at the leaves also allows us to maintain interpretability, as linear controllers are widely used and generally considered interpretable for humans \citep{Hein2020InterpretableCB}. However, for large feature spaces typically encountered in real-world problems, such a controller would not be interpretable.  \textcolor{black}{As such, in our work, we utilize \emph{sparse} linear controllers at the leaves to balance the trade-off between sparsity/complexity in logic, model depth, and performance.} 


\subsection{\textcolor{black}{ICCT Key Elements}}
In this section, we discuss our ICCT's interpretable procedure for determining an action given an input feature. As our ICCT configuration maintains interpretability both during training via RL and deployment, the inference of an action \textcolor{black}{must allow gradient flow. We present a novel approach that allows for direct optimization of sparse logical models via an online differentiable crispification procedure to determine feature importance (Section \ref{sec:decision_nodes_crispification}) and allows for bifurcate decisions (Section \ref{sec:decision_outcome_crispification}).} In Algorithm \ref{alg:algorithm_training}, we provide general pseudocode representing our ICCT's decision-making process. 

At each timestep, the ICCT model, $I(\cdot)$, receives a state feature, $\vec{x}$. To determine an action in an interpretable form, in Steps 1 and 2 of Algorithm \ref{alg:algorithm_training}, we start by applying the differentiable crispification approaches of \textsc{Node$\_$Crisp} and \textsc{Outcome$\_$Crisp} to decision nodes so that each decision node is only conditioned upon a single variable (Section \ref{sec:decision_nodes_crispification}), and the evaluation of a decision node results in a Boolean (Section \ref{sec:decision_outcome_crispification}). Once the operations are completed, in Step 3, we can utilize the input feature, $\vec{x}$, and logically evaluate each decision node until arrival at a linear leaf controller (\textsc{Interpretable$\_$Node$\_$Routing}). The linear leaf controller is then modified, in Step 4, to only condition upon $e$ features, where $e$ is a sparsity parameter specified a priori (Section \ref{sec:sparse_sub_controllers}). Finally, an action can be determined via sampling from a Gaussian distribution conditioned upon the mean generated via the input-parameterized sparse leaf controller, $l^*_d$, and scalar variance maintained within the leaf, $\gamma_d$, during training (Step 6) or directly through the outputted mean (Step 8) during runtime.

\begin{algorithm}[ht]
\caption{ICCT Action Determination}
\label{alg:algorithm_training}
\textbf{Input}: \small ICCT $I(\cdot)$, state feature $\vec{x} \in S$, controller sparsity $e$, training flag $t \in \textsc{\{True, False\}}$ \\
\textbf{Output}: action $a \in \mathbb{R}$ 
\begin{algorithmic}[1] 
\STATE \textsc{Node$\_$Crisp}($\sigma(\alpha(\vec{w}^T_i \vec{x} - b_i))$)
$\rightarrow \sigma(\alpha(w_i^k x^k - b_i))$
\STATE \textsc{Outcome$\_$Crisp}($\sigma(\alpha(w_i^k x^k - b_i))$)
$\rightarrow \mathbbm{1}(\alpha(w_i^k x^k - b_i)>0)$
\STATE $l_d \leftarrow$ \textsc{Interpretable$\_$Node$\_$Routing}($\vec{x}$)
\STATE $l_d^* \leftarrow$ \textsc{Enforce$\_$Controller$\_$Sparsity}($e$, $l_d$)
\IF {$t$}
\STATE \textcolor{black}{$a\sim \mathcal{N}(l_d^*(\vec{x}), \gamma_d)$}
\ELSE
\STATE $a\leftarrow l_d^*(\vec{x})$
\ENDIF 
\end{algorithmic}
\end{algorithm}

\subsubsection{Decision Node Crispification}
\label{sec:decision_nodes_crispification}
The \textsc{Node$\_$Crisp} procedure in Algorithm \ref{alg:algorithm_training} \textcolor{black}{recasts} each decision node to split upon a single dimension of $\vec{x}$.
Instead of using a non-differentiable argument max function as in \citet{Silva2021EncodingHD} to determine the most impactful feature dimension, we utilize a softmax function, also known as softargmax \citep{Goodfellow-et-al-2016}, described by Equation \ref{eq:softmax}. In this equation, we denote the softmax function as $f(\cdot)$, which takes as input a set of class weights and produces class probabilities. \textcolor{black}{Here, $\vec{w}_i$ represents a categorical distribution with class weights, individually denoted by $w_i^j$}, and $\tau$ is the temperature, determining the steepness of $f(\cdot)$. 
\textcolor{black}{
\begin{equation}
\label{eq:softmax}
    f(\vec{w}_i)_k = \frac{\exp{\big(\frac{w_i^k}{\tau}\big)}}{\sum_{j=1}^m \exp{\big(\frac{w_i^j}{\tau}\big)}} 
\end{equation}
}

\noindent While setting the temperature near-zero would satisfy our objective of producing a one-hot vector, where the outputted class probability of the index of the most impactful feature would be one, this operation can lead to a large variance within the gradients and unstable training. We therefore set the softmax temperature, $\tau$ equal to 1, which we find effective empirically, and utilize a differentiable $\text{one\_hot}$ function, $g(\cdot)$, to produce a one-hot vector with the element associated with the highest-weighted class set to one and all other elements set to zero. We display the procedure for determining the one-hot vector associated with the largest magnitude in Equation \ref{eq:crispification}. 
\begin{equation}
\label{eq:crispification}
    \vec{z_i} = g(f(|\vec{w_i}|))
\end{equation}
Here, $|\vec{w_i}|$ represents a vector with absolute elements within $\vec{w_i}$. \emph{We maintain differentiability in the procedure described in Equation \ref{eq:crispification} by utilizing the straight-through trick \citep{Bengio2013EstimatingOP}}. This allows us to obtain the desired output, $\vec{z}_i$, a one-hot vector over weights required for the purpose of matching the decision node representation of a decision tree, while maintaining gradients for all weight parameters $\{w_i^j\}_{j=1}^m$ (by treating the gradient with respect to $\vec{z_i}$ as the gradient with respect to $f(|\vec{w_i}|)$). This procedure is further elaborated in Section \ref{sec:diff_argmax} and Algorithm \ref{alg:argmax}. 
\color{black}

The one-hot encoding $\vec{z_i}$ can then be element-wise multiplied by the original weights to produce a new set of weights with only one active weight, $\vec{z}_i\circ \vec{w}_i\rightarrow\vec{w}_i'$. Accordingly, the decision node representation is transferred from $\sigma(\alpha(\vec{w}^T_i \vec{x} - b_i)) \to  \sigma(\alpha(\vec{w}_{i}'^{T} \vec{x} - b_i))=\sigma(\alpha(w_i^k x^k - b_i))$, where $k$ is the index of the most impactful feature. 
We provide an algorithm detailing the \textsc{Node$\_$Crisp} procedure in Algorithm \ref{alg:node_crisp}, where node crispification recasts each decision node to split upon a single dimension of the input.

\begin{algorithm}[H]
\caption{\textcolor{black}{Node Crispification: \textsc{Node$\_$Crisp}$(\cdot)$}}
\label{alg:node_crisp}
\textbf{Input}: \small The original fuzzy decision node $\sigma(\alpha(\vec{w}^T_i \vec{x} - b_i))$, where $i$ is the decision node index, $\vec{w}_i=[w_i^1, w_i^2, ..., w_i^j, w_i^{j+1}, ..., w_i^m]^T$, and $m$ is the number of input features \\
\textbf{Output}: \small The intermediate decision node representation $\sigma(\alpha(w_i^k x^k - b_i))$ (see the green box in Figure \ref{fig:node_outcome_crisp}) 
\begin{algorithmic}[1] 
\STATE $\vec{z}_i=$ \textsc{diff\_argmax}$(|\vec{w}_i|)$ (\textsc{diff\_argmax}$(\cdot)$ displayed in Algorithm \ref{alg:argmax})
\STATE $\vec{w}_i' = \vec{z}_i\circ \vec{w}_i$
\STATE $\sigma(\alpha(w_i^k x^k - b_i)) = \sigma(\alpha(\vec{w}_i'^T \vec{x} - b_i))$ 
\end{algorithmic}
\end{algorithm}

Node crispification takes as input the original fuzzy decision node, $\sigma(\alpha(\vec{w}^T_i \vec{x} - b_i))$, where \emph{all} input features are used in determining the output of decision node $i$. The output of this function is an intermediate decision node, $\sigma(\alpha(w_i^k x^k - b_i))$, where the output of decision node $i$ is only determined by \emph{a single }feature, $x^k$. To perform this transformation, in Line 1, we use the differentiable argument max function (in Algorithm \ref{alg:argmax}) to produce a one-hot vector, $\vec{z}_i$, with the element associated with the most impactful feature set to one and all other elements set to zero. In Line 2, we element-wise multiply the one-hot encoding, $\vec{z_i}$, by the original weights, $\vec{w}_i$, to produce a new set of weights with only one active weight, $\vec{w}_i'$. In Line 3, we show that by multiplying $\vec{x}$ by $\vec{w}_i'$, we can obtain the intermediate decision node $\sigma(\alpha(w_i^k x^k - b_i))$, where $k$ is the index of the most impactful feature (i.e., $k=\argmax_j(|w_i^j|)$). The transformation to each decision node performed by node crispification is further displayed by the green arrow in Figure \ref{fig:node_outcome_crisp}.

\begin{figure}[ht]
\centering
\includegraphics[width=0.36\textwidth]{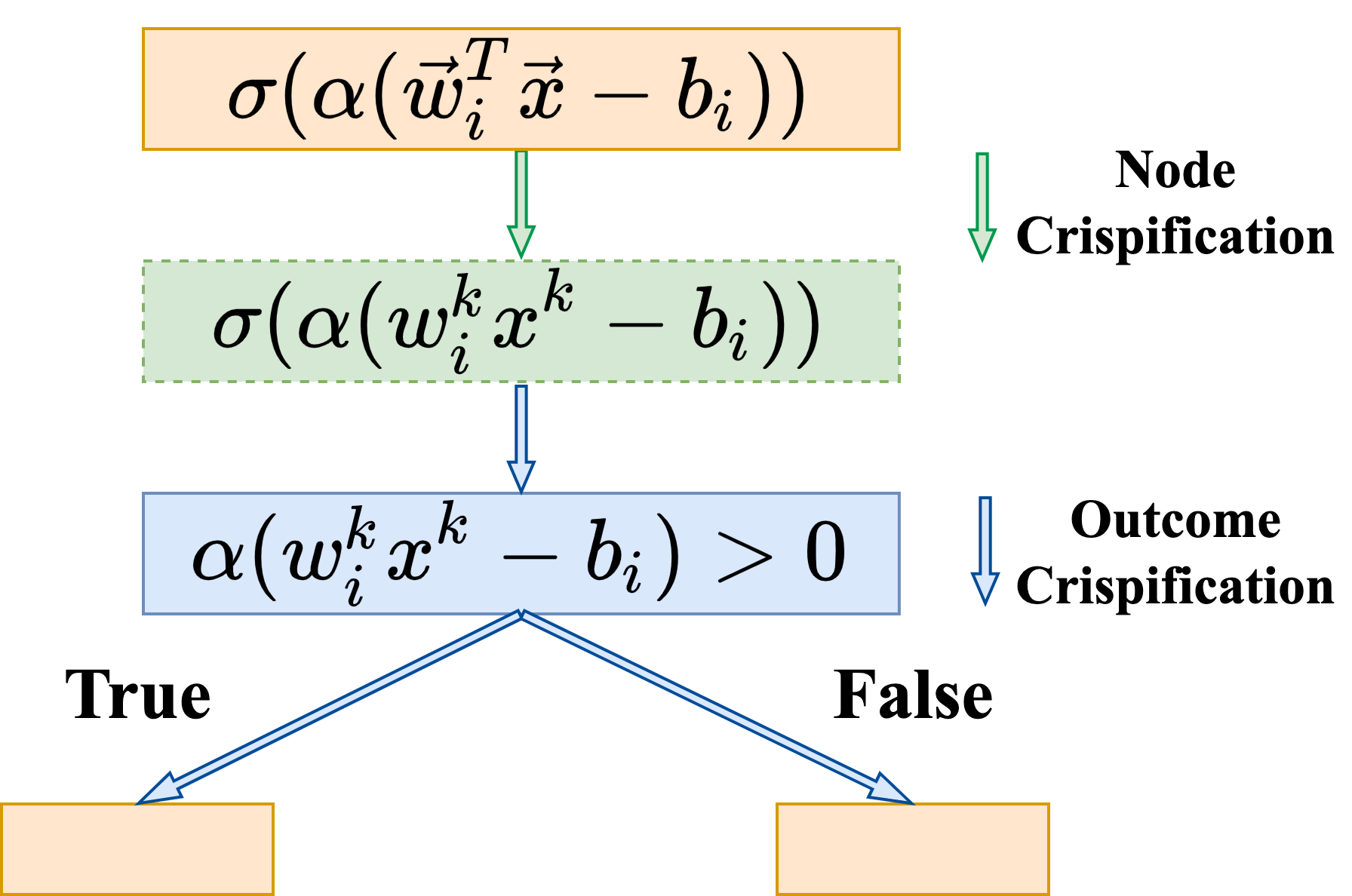}
\caption{This figure displays the process of differentiable crispification, including node crispification (Algorithm \ref{alg:node_crisp}) and outcome crispification (Algorithm \ref{alg:outcome_crisp}). The node crispification sparsifies the weight vector, $\vec{w_i}$, and chooses the most impactful feature. The outcome crispification enforces a ``hard'' decision at the node rather than a ``soft'' decision, so the computation proceeds along one branch. Both operations are differentiable through the use of the straight-through trick. }
\label{fig:node_outcome_crisp}
\end{figure}

\noindent Below, we conduct a short example detailing our procedure.

\color{black}
\textit{Example:} Consider we have a two-leaf decision tree (one decision node) with an input feature, $\vec{x}=[2,3]^T$ with a cardinality of 2 (i.e., $m = 2$), associated weights of $\vec{w_1} = [2,1]^T$, and a bias term $b_1$ of 1. The sigmoid steepness, $\alpha$, is also set equal to 1 for simplicity. It is easily seen that the most impactful weight within the decision node is $w_1^1=2$. Utilizing Equation \ref{eq:crispification}, we can compute $\vec{z_1}=[1,0]^T$. Multiplying $\vec{z_1}$ to the original weights, $\vec{w_1}$, and input feature, $\vec{x}$, subtracting $b_1$, and scaling by $\alpha$, we have an crisp decision node $\sigma(2x^1-1)$ or $\sigma(3)=0.95$. Here, $0.95$ is the probability that the decision node evaluates to \textsc{True}. We display a depiction of this example in the left-hand side of Figure \ref{fig:example_diff}.

\subsubsection{Decision Outcome Crispification}
\label{sec:decision_outcome_crispification}
Here, we describe the second piece of our online differentiable crispification procedure, noted as \textsc{Outcome$\_$Crisp} in Algorithm \ref{alg:algorithm_training}.
\textsc{Outcome$\_$Crisp} translates the outcome of a decision node so that the outcome is a Boolean decision rather than a probability generated via a sigmoid function (i.e., $p =y_i$ for True/Left Branch and $q=1-y_i$ for False/Right Branch). We start by creating a soft vector, $\vec{v}_i = [\alpha(w_i^k x^k - b_i), 0]$, for the $i^{th}$ decision node. Placing $\vec{v}_i$ through a softmax operation, we can produce the probability of tracing down the left branch or right. 
We can then apply the differentiable $\text{one-hot}$ function, $g(\cdot)$,
to produce a hard decision of whether to branch left or right, denoted by $y_i$ and described by Equation \ref{eq:decision_outcome_crispfication}. 
\begin{equation}
    \label{eq:decision_outcome_crispfication}
    [y_i,1-y_i] = g(f(\vec{v}_i)) 
\end{equation}
Essentially, the decision node will evaluate to \textsc{True} if $\alpha(w_i^k x^k - b_i)>0$ and right otherwise. This process can be expressed as an indicator function $ \mathbbm{1}(\alpha(w_i^k x^k - b_i)>0)$. 

We note the procedure of $g(f(\vec{v}_i))$ is highly similar to that in Equation \ref{eq:crispification}, both outputting a one-hot vector, with the former input being the decision node weights, $|\vec{w}_i|$, and the latter input being the soft vector representation of the decision node outcome, $\vec{v}_i$.
We provide an algorithm detailing the \textsc{Outcome$\_$Crisp} procedure in Algorithm \ref{alg:outcome_crisp}, where outcome crispification translates the outcome of a soft decision node to a hard decision node, resulting in a Boolean output from the decision node rather than a set of probabilities.

\begin{algorithm}[H]
\caption{\textcolor{black}{Outcome Crispfication: \textsc{Outcome$\_$Crisp}$(\cdot)$}}
\label{alg:outcome_crisp}
\textbf{Input}: \small The intermediate decision node $\sigma(\alpha(w_i^k x^k - b_i))$,  where $i$ is the decision node index, $k=\arg\max_j(|w_i^j|)$, and $w_i^j$ is the $j$th element in $\vec{w}_i$ \\
\textbf{Output}: \small Crisp decision node $\mathbbm{1}(\alpha(w_i^k x^k - b_i)>0)$ (see the blue box in Figure \ref{fig:node_outcome_crisp}) 
\begin{algorithmic}[1] 
\STATE $\vec{v}_i = [\alpha(w_i^k x^k - b_i), 0]$
\STATE $\vec{z}_i'=$ \textsc{diff\_argmax}$(\vec{v}_i)$ (\textsc{diff\_argmax}$(\cdot)$ displayed in Algorithm \ref{alg:argmax})
\STATE $\mathbbm{1}(\alpha(w_i^k x^k - b_i)>0) = \vec{z}_i'[0]$
\end{algorithmic}
\end{algorithm}

Outcome crispification takes in the intermediate decision node $\sigma(\alpha(w_i^k x^k - b_i))$, which outputs the probability of branching left. The output of \textsc{Outcome$\_$Crisp} is the crisp decision node, $\mathbbm{1}(\alpha(w_i^k x^k - b_i)>0)$, a Boolean decision to trace down to the left branch OR right. In Line 1, we construct a soft vector representation of the decision node $i$'s output, $\vec{v}_i$, by concatenating $\alpha(w_i^k x^k - b_i)$ with a $0$. In Line 2, we use the differentiable argument max function (in Algorithm \ref{alg:argmax}) to produce a one-hot vector, $\vec{z}_i'$, where the first element represents the Boolean outcome of the decision node. In Line 3, we show that the output of the crisp decision node, $\mathbbm{1}(\alpha(w_i^k x^k - b_i)>0)$, can be obtained by choosing the first element of vector $\vec{z}_i'$ (we use bracket indexing notation here, starting with zero).  We further display the transformation performed by outcome crispification by the blue arrows in Figure \ref{fig:node_outcome_crisp}. 

\color{black}
\textit{Example: Continuing the example in Section \ref{sec:decision_nodes_crispification}}, we can take the outputted crisp decision node and generate a vector $\vec{v}_1 =[2x^1-1,0]^T$, or by substituting the input feature, $\vec{v}_1 =[3,0]^T$. Performing the operations specified in Equation \ref{eq:decision_outcome_crispfication}, we receive the intermediate output from the softmax $[0.95, 0.05]^T$ (rounded to two decimal places), the one-hot vector $[1,0]^T$ after performing the $\text{one\_hot}$ operation, and finally $y_1 = 1$, denoting that the decision-tree should follow the left branch.  We display a depiction of this example on the right-hand side of Figure \ref{fig:example_diff}.

\begin{figure}[ht]
\centering
\includegraphics[width=\textwidth]{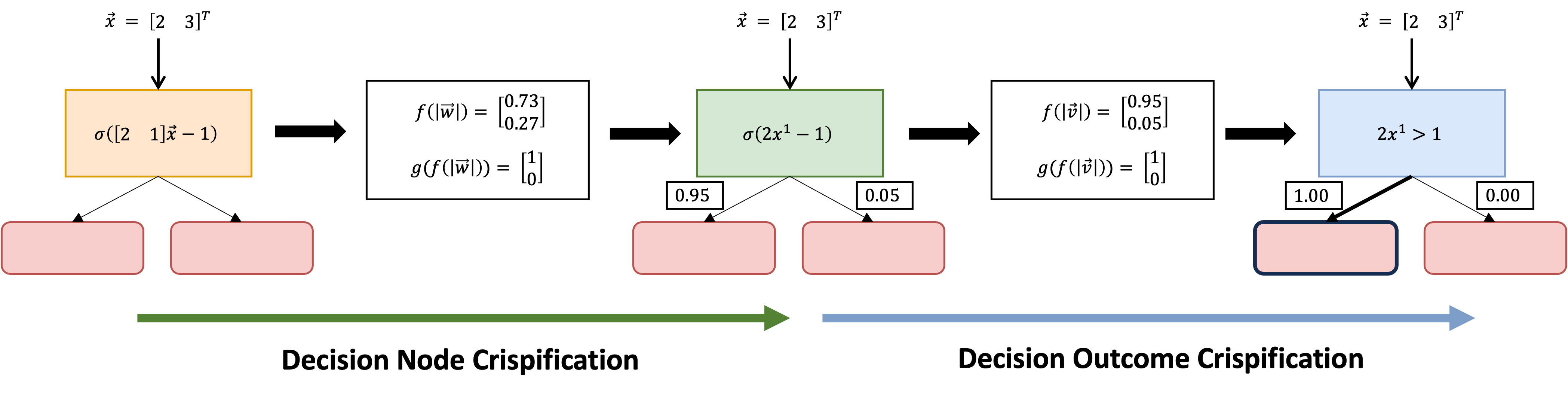}
\caption{This figure displays the process of decision node crispification and decision outcome crispification across the Examples within Section \ref{sec:decision_nodes_crispification} and Section \ref{sec:decision_outcome_crispification}.}
\label{fig:example_diff}
\end{figure}

\textit{Conversion to a Simple Form:}
The above crispification processes produce decision-tracing equal to that of a DT. The node representation can thus be simplified to that of Figure \ref{fig:dt} by algebraically reducing each crisp decision node to $ x^k>\frac{b_i}{w_i^k}$ (given $\alpha w_i^k>0$) or $ x^k<\frac{b_i}{w_i^k}$ (given $\alpha w_i^k<0$).

\subsubsection{Sparse Linear Leaf Controllers}
\label{sec:sparse_sub_controllers}
After applying the decision node and outcome crispification to all decision nodes and outcomes, the decision can be routed to leaf node (Step 3). This section describes the procedure to translate a linear leaf controller to condition upon $e$ features (\textsc{Enforce$\_$Controller$\_$Sparsity} procedure in Algorithm \ref{alg:algorithm_training}), enforcing sparsity within the leaf controller and thereby, enhancing ICCT interpretability. As noted in Section \ref{sec:icct_arch}, our ability to utilize sparse sub-controllers allows us to balance between interpretability and performance. The sparsity of the linear sub-controllers ranges from setting $e=0$ and maintaining static leaf distributions, where each leaf node contains scalar value representing the mean (i.e., ICCT-static in Section \ref{sec:results}), to $e=m$, containing a linear controller parameterized by the entire feature space of $\vec{x}$ (i.e., ICCT-complete in Section \ref{sec:results}).
Equation \ref{eq:sparse_sub_impact} displays the procedure for determining a $k$-hot encoding, $\vec{u}_d$, that represents the $k$ \textcolor{black}{(or in our case, $e$)} most impactful selection weights within a leaf's linear controller. The $k$-hot function, denoted by $h(\cdot)$, takes as input a vector of weights and returns an equal-dimensional vector with $k$ elements set to one. The indexes associated with the elements set to one match the $k$ highest-weighted elements within the input feature.
\begin{equation}
    \label{eq:sparse_sub_impact}
    \vec{u}_d = h (f(|\vec{\theta}_d|))
\end{equation}

\noindent Here, $|\vec{\theta}_d|$ represents a vector with absolute elements within the per-leaf selector weights, $\vec{\theta}_d$. As before, we maintain differentiability and formulate a differentiable top-$k$ function in Equation \ref{eq:sparse_sub_impact} by utilizing the straight-through trick \citep{Bengio2013EstimatingOP} and iteratively applying \textsc{diff\_argmax}$(\cdot)$ (Algorithm \ref{alg:argmax}) $k$ times.
\textcolor{black}{In Equation \ref{eq:translate_to_sparse_submodel}, we transform a fuzzy leaf node, for leaf, $l_d$ (represented in Equation \ref{eq:complete_leaf}), into a sparse linear sub-controller, $l_d^*$, with the sparse feature selection vector, $\vec{u_d}$, given by Equation~\ref{eq:sparse_sub_impact}.}
\begin{equation}
    \label{eq:translate_to_sparse_submodel}
    l_d^* \triangleq (\vec{u}_d \circ \vec{\beta}_d)^T (\vec{u}_d \circ \vec{x}) + \vec{u}_d^T \vec{\phi}_{d}
\end{equation}
A depiction of the sparse sub-models can be seen at the bottom right-hand side of Figure \ref{fig:dt}, where the sparsity of the sub-controllers, $e$, is set to 1 and the dimension of the action space is 2.

\subsubsection{Differentiable Argument Max Function for Differentiable Crispification}
\label{sec:diff_argmax}
In this section, we provide a description of the \textcolor{black}{differentiable argument max} function. 

\begin{algorithm}[H]
\caption{Differentiable Argument Max Function for Crispification\textcolor{black}{: \textsc{diff\_argmax}$(\cdot)$}}
\label{alg:argmax}
\textbf{Input}: \small Logits $\vec{q}$ \\
\textbf{Output}: \small One-Hot Vector $\vec{h}$ 
\begin{algorithmic}[1] 
\STATE $\vec{h}_{soft} \leftarrow f(\vec{q})$
\STATE $\vec{h}_{hard} \leftarrow \textsc{one\_hot}(\textsc{argmax}(f(\vec{q})))$ \COMMENT{step 1 for function $g(\cdot)$}
\STATE $\vec{h} = \vec{h}_{hard} +\vec{h}_{soft} - $\textsc{stop\_grad}($\vec{h}_{soft}$) \COMMENT{step 2 for function $g(\cdot)$}
\end{algorithmic}
\end{algorithm}

Similar to \citep{Hafner2021MasteringAW}, we present a function call (in Algorithm \ref{alg:argmax}) that can be utilized to maintain gradients over a non-differentiable argument max operation. The function takes in a set of logits, $\vec{q}$, and applies a softmax operation, denoted by $f(\cdot)$, to output $\vec{h}_{soft}$, as shown in Line 1. In Line 2, the logits are transformed using an argument max followed by a one-hot procedure, causing the removal of gradient information, producing $\vec{h}_{hard}$. In Line 3, we combine $\vec{h}_{soft}$, $\vec{h}_{hard}$, and \textsc{stop\_grad}($\vec{h}_{soft}$) to output $\vec{h}$, where \textsc{stop\_grad}($\cdot$) keeps the values and removes the gradient data of $\vec{h}_{soft}$. The outputted value of $\vec{h}$ is equal to that of $\vec{h}_{hard}$. However, the gradient maintained within $\vec{h}$ is associated with $\vec{h}_{soft}$. Automatic differentiation frameworks can then utilize the outputted term to perform backpropagation. Here, the operations in Line 2 and Line 3 compose function $g(\cdot)$ in Equation \ref{eq:crispification} and \ref{eq:decision_outcome_crispfication}.

\textcolor{black}{\noindent \textbf{Summary:} In this section, we discuss our novel interpretable reinforcement learning architecture, ICCTs. We present a description of components of ICCTs, including decision nodes and linear leaf controllers, and provide a differentiable crispification procedure allowing for optimization similar to a differentiable decision tree (DDT) while maintaining a forward-propagation process identical to a sparse decision tree. To the best of our knowledge, we present the first truly interpretable tree-based framework for continuous control.}

\section{Universal Function Approximation}
\label{sec:ufa}
In this section, we provide a proof to show our ICCTs are universal function approximators, that is, can represent any decision surface given enough parameters. Our ICCT architecture consists of successive indicator functions, whose decision point lies among a single dimension of the feature space, followed by a linear controller to determine a continuous control output. For simplicity, we assume below that the leaf nodes contain static distributions. However, maintaining a linear controller at the leaves is more expressive and thus, the result below generalizes directly to ICCTs. 

The decision-making of our ICCTs can be decomposed as a sum of products. In Equation \ref{eq:decomp}, we display a computed output for a 4-leaf tree, where decision node outputs, $y_i$, are given by Equation \ref{eq:ddt_decision_node}. Here, the sigmoid steepness, $\alpha$ is set to infinity (transforming the sigmoid function into an indicator function) resulting in hard decision points ($y_i \in \{0,1\}$). Equation \ref{eq:decomp} shows that the chosen action is determined by computation of probability of reaching a leaf, $y$, multiplied by static tree weights maintained at the distribution, $p$.
\begin{align}
\small
    \label{eq:decomp}
    ICCT(x) &= p_1 (y_1 *y_2) + p_2 (y_1 * (1-y_2)) \\ \nonumber
    & + p_3 ((1-y_1) * y_3) + p_4*((1-y_1) *(1-y_3))
\end{align}
\normalsize
Equation \ref{eq:decomp} can be directly simplified into the form of $G(x) = \sum_{j=1}^N p_j \sigma (w_j^T x + b_j)$, similar to Equation 1 in \citep{Cybenko1989ApproximationBS}. \citep{Cybenko1989ApproximationBS} demonstrates that finite combination of fixed, univariate functions can approximate any continuous function. The key difference between our architecture is that our univariate function is an indiator function rather than the commonly used sigmoid function. Below, we provide two lemmas to show that indicator functions fall within the space of univariate functions \citep{Cybenko1989ApproximationBS}. 
\begin{lemma}
An indicator function is sigmoidal.
\end{lemma}
\textit{Proof:} This follows from the definition of sigmoidal: $\sigma(t)\to 1$ as $t \to \infty$ and $\sigma(t)\to 0$ as $t \to -\infty$.
\begin{lemma}
An indicator function is discriminatory.
\end{lemma}
\textit{Proof:} As an indicator function is bounded and measureable, by Lemma 1 of \citep{Cybenko1989ApproximationBS}, it is discriminatory.

\begin{theorem}
\label{thm:ufa}
Let $\sigma$ be any continuous discriminatory function. ICCTs are universal function approximators, that is, dense in the space of $C(I_n)$. 
In other words, there is a representation of ICCTs, $I(x)$, for which $|I(x) - f(x)| < \epsilon$ for all $x \in I_n$, for any function, f ($f \in C(I_n)$), where $C(I_n)$ denotes the codomain of an n-dimensional unit cube, $I_n$.
\end{theorem}
\begin{proof}
As the propositional conditions hold for Theorem 1 in \citep{Cybenko1989ApproximationBS}, the result that ICCTs are dense in $C(I_n)$ directly follows. We note that as the indicator function 
jump-continuous, we refer readers to \citep{Selmic2002NeuralnetworkAO} whom extend UFA for $G(x) = \sum_{j=1}^N p_j \sigma (w_j^T x + b_j)$ to the case when $\sigma$ is jump-continuous. As ICCTs are dense in $C(I_n)$, ICCTs are universal function approximators.
\end{proof}

\section{Model Robustness Verification}
\label{sec:verification}
A desirable property for a controller is to be able to verify the policy, ensuring the controller outputs desirable values for a set of inputs. This often translates to answering the following question: \begin{itemize}
    \item \textit{For a range of input features, what is the range of output values that can be expected?}
\end{itemize} 
By answering this question, engineers and end-users can attain key insights into a policy's decision-making behavior and make guarantees about its behavior. Utilizing autonomous driving as an example, an engineer may want to verify that if a human is detected within $5$ meters, the acceleration of the vehicle is never above $5m/s^{-2}$. Verification of policies is vital in creating models that are safe and can help ensure that models accurately perform the purpose they are designed for.


In this section, we analyze the time complexity of verifying an ICCT. Following \cite{Chen2019RobustnessVO}, we formalize the problem of \textit{robustness verification} as follows: Consider a regression model: $f: \mathbb{R}^d \to \mathbb{R}$, where $d$ is the dimension of the input features and the output is a real-valued scalar. In \cite{Chen2019RobustnessVO}, for input, $x$, the minimal adversarial perturbation is defined by Equation~\ref{eq:min_adversarial_perturbation}, where $y=f(x)$ is the expected controller output value. The solution to this equation determines the minimum input perturbation required to have the controller output a value different from the expected value, $y$. 
\begin{equation}
\label{eq:min_adversarial_perturbation}
    r^* = \min_\lambda ||\lambda||_\infty \text{ s.t. } f(x+\lambda) \neq y
\end{equation}
As we are concerned with \textit{continuous} control policies, where a slight perturbation on the input may cause a slight perturbation on the controller output, we instead generalize $y$ to an expected output range, $y=\{a,b\}$, and search for an input perturbation that causes $y$ to fall out of the specified range, i.e., $f(x+\lambda)\notin \{a,b\}$. Following our autonomous driving example, finding the minimum adversarial perturbation allows an engineer to understand what input deviations may result in unsafe or undesirable controller outputs. Thus below, we present a discussion of the time complexity associated with solving Equation~\ref{eq:min_adversarial_perturbation} of neural network models and our ICCTs. \emph{Positively, for our ICCTs, determining the minimal input perturbation can be done in linear time.}


Due to complex nonlinearities within neural network architectures, small input perturbations can lead to large changes in predicted values \citep{Szegedy2013IntriguingPO, Carlini2016TowardsET}, making it intractable to perform verification. Often, verification is only possible if the neural network architecture follows certain desiderata \citep{Katz2017ReluplexAE} or by utilizing convex relaxations of nonlinear activation functions \citep{Salman2019ACR}. Even so, such verification, in the best case, is only NP-Complete.

ICCTs, on the other hand, can be verified in linear time. Here, we first present an analysis showing ICCT-static (ICCT maintaining static Gaussian distributions at each leaf) can be verified in linear time. We then extend this proof to ICCT-complete, where linear controllers parameterized by the input features are maintained at each leaf. For simplicity, we assume that the output value of our controller is solely determined by the mean value within the leaf distribution of our ICCTs, similar to how ICCTs are deployed during runtime. 

Assume we have a decision tree that has $n$ decision nodes and $n+1$ leaf nodes. For each decision node $i$, we define a variable, $t_i$, representing which feature is activated within the decision node, and a variable, $\eta_i$, representing the threshold maintained with the decision node (which is equal to $\frac{b_i}{w_i}$). Depending on the outcome of the decision node, the computation will further proceed to either the left or right child until arrival at a leaf node, where the leaf node contains a scalar value representing the Gaussian mean.

Following \cite{Chen2019RobustnessVO}, the key of the proof is to split the $d$-dimensional input space to $n+1$ hyperspaces corresponding to leaf nodes, such that any input will result in falling into one and only one leaf node's hyperspace. This can be done by traversing the entire tree and computing bounding boxes via a depth-first search. By definition, all input variables will reach the root node, resulting in a root node box represented by the Cartesian product $B = [-\infty, \infty] \times \cdots [-\infty, \infty]$, of cardinality $d$. Each child's box can be obtained by splitting one interval from the parent box based on the split condition (the variable selected, $t_i$, and threshold, $\eta_i$). This process can be completed until the entire tree is traversed (via a depth-first search fashion), resulting in a time complexity of $O(nd)$, where $n$ is the number of nodes and $d$ is the cardinality of the input space.

The distance from an input $x$ to a leaf node's region can be written as a vector, $\epsilon(x, B^i) \in \mathbb{R}^d$, defined by the Equation \ref{eq:cases}, where $l_t^i, r_t^i$, represent the upper and lower bound of a node's Cartesian product on dimension $t$ for leaf $i$, respectively.
\begin{equation}
\label{eq:cases}
  \epsilon(x, B^i)_t =
    \begin{cases}
      0 & \text{if $i_t \in (l_t^i,r_t^i)$}\\
      x_t-r_t^i & \text{if $x_t > r_t^i$}\\
      l_t^i-x_t & \text{if $x_t \leq l_t^i$}
    \end{cases}       
\end{equation}

Thus, the minimal distortion required to result in an incorrect output value can be obtained by Equation \ref{eq:min_distortion}, where $y=\{a,b\}$ is the output range desired, and $v_i$ is the value for leaf node i. Intuitively, Equation~\ref{eq:min_distortion} finds the minimum perturbation on a feature, $t$, that leads to a leaf node with associated output outside of the desired output range. 
\begin{equation}
\label{eq:min_distortion}
    r^* = \min_{i:v_i \notin y, t\in[d]} \epsilon(x, B^i)_t
\end{equation} 

The time complexity of the verification algorithm for ICCT robustness is $O(nd)$ due to the traversal of the tree, the combination of bounding boxes, and the minimal perturbation finding in Equation~\ref{eq:min_distortion} by iterating over all leaf node and all feature dimensions. As stated above, this results in addressing the \textit{decision problem of robustness verification} in linear time.

When extending to ICCTs with linear controllers at the leaves, we can utilize the same formalism to obtain the bounding boxes represented by Cartesian products at each leaf. However, as each leaf controller depends on input variables, the range of outputs would require extra computation based on the bounding boxes and linear controller parameters. Because of the monotony of the linear controllers with respect to each input feature, the computation is still $O(d)$ for each leaf node, and therefore we can still achieve the overall $O(nd)$ time complexity for the \textit{robustness verification}. 



\color{black}

\begin{sidewaystable*}
\resizebox{\textwidth}{!}{%
\begin{tabular}{@{}ccccccc@{}}
\multicolumn{1}{c}{\large Worst to Best:} & \multicolumn{6}{l}{\Mark{295}} \\
\specialrule{.2em}{.1em}{.1em} 
\multirow{2}{*}{Method} & \multicolumn{2}{c|}{Common Continuous Control Problems} & \multicolumn{4}{c}{\textcolor{black}{Autonomous Driving Problems}} \\ \cmidrule(l){2-7} 
& Inverted Pendulum & \multicolumn{1}{c|}{Lunar Lander} & Lane Keeping & Single-Lane Ring & Multi-Lane Ring & Figure-8\\ \specialrule{.2em}{.1em}{.1em} \multirow{2}{*}{DT} & \cellcolor[HTML]{ff9023}$155.0\pm0.9$& \cellcolor[HTML]{ff9833}$-285.5\pm15.6$  &\cellcolor[HTML]{ff8f21}$-359.0 \pm 11.0$ &  \cellcolor[HTML]{fff4ea}\textbf{123.2 $\pm$ 0.03} &\cellcolor[HTML]{ff7f00}$503.2\pm24.8$ & \cellcolor[HTML]{ffcc9d}$831.1\pm1.1$  \\   & \cellcolor[HTML]{ffb875}256 leaves (766 params)  & \cellcolor[HTML]{ffb875}256 leaves (1022 params) & \cellcolor[HTML]{ffb977}256 leaves (766 params) &  \cellcolor[HTML]{fff3e8}32 leaves (94 params)  & \cellcolor[HTML]{ffdec0}256 leaves (1022 params) & \cellcolor[HTML]{ffc996}256 leaves (766 params)  \\  \midrule \multirow{2}{*}{DT w$\backslash$ DAgger} & \cellcolor[HTML]{ffd8b4}$776.6\pm54.2$& \cellcolor[HTML]{ffe1c5}$184.7\pm17.3$  &\cellcolor[HTML]{ffe9d5}\textbf{395.2 $\pm$ 13.8} & \cellcolor[HTML]{fff0e3} 121.5 $\pm$ 0.01 &\cellcolor[HTML]{ffeede}$1249.4\pm3.4$ &  \cellcolor[HTML]{fff4ea}\textbf{1113.8 $\pm$ 9.5}  \\   & \cellcolor[HTML]{ffeede}32 leaves (94 params)  & \cellcolor[HTML]{ffefe1}32 leaves (126 params) & \cellcolor[HTML]{fff4ea}16 leaves (46 params) &  \cellcolor[HTML]{fff4ea}16 leaves (46 params)  &  \cellcolor[HTML]{fff4ea}31 leaves (122 params) &  \cellcolor[HTML]{fff4ea}16 leaves (46 params)  \\  \midrule   \multirow{2}{*}{CDDT-Crisp} & \cellcolor[HTML]{ff7f00}$5.0\pm0.0$ & \cellcolor[HTML]{ff7f00}$-451.6\pm97.3$ &\cellcolor[HTML]{ff7f00}$-43526.0\pm15905.0$ &\cellcolor[HTML]{ff7f00}$68.1\pm18.7$ &\cellcolor[HTML]{ff962f}$664.5\pm192.6$ & \cellcolor[HTML]{ff8813}$322.9\pm47.1$  \\   & \cellcolor[HTML]{fff4ea}2 leaves (5 params) & \cellcolor[HTML]{fff4ea}8 leaves (37 params) & \cellcolor[HTML]{fff3e8}16 leaves (61 params) & \cellcolor[HTML]{fff4ea}16 leaves (61 params) & \cellcolor[HTML]{fff4ea}16 leaves (77 params) &  \cellcolor[HTML]{fff4ea}16 leaves (61 params)  \\  \midrule  \multirow{2}{*}{ICCT-static} & \cellcolor[HTML]{fff2e5}\textbf{984.0 $\pm$ 10.4} & \cellcolor[HTML]{ffe2c7}\textbf{192.4 $\pm$ 10.7} & \cellcolor[HTML]{ffe7d0}374.2$\pm$55.8 &\cellcolor[HTML]{ffeddc}$120.5\pm0.5$ & \cellcolor[HTML]{fff2e5}\textbf{1271.7 $\pm$ 4.1} & \cellcolor[HTML]{ffe4cc} 1003.8$\pm$27.2  \\   & \cellcolor[HTML]{ffead7}32 leaves (125 params) & \cellcolor[HTML]{ffeddc}32 leaves (157 params) & \cellcolor[HTML]{fff3e8}16 leaves (61 params) & \cellcolor[HTML]{fff4ea}16 leaves (61 params) & \cellcolor[HTML]{fff4ea}16 leaves (77 params) & \cellcolor[HTML]{fff4ea}16 leaves (61 params) \\  \specialrule{.2em}{.1em}{.1em}   \multirow{2}{*}{ICCT-1-feature} & \cellcolor[HTML]{fff4ea}$1000.0\pm0.0$ & \cellcolor[HTML]{ffe2c7}$190.1\pm13.7$ &\cellcolor[HTML]{ffeede}$437.6\pm7.0$ &\cellcolor[HTML]{fff0e3}$121.6\pm0.5$ &\cellcolor[HTML]{fff2e5}$1269.6\pm10.7$ &\cellcolor[HTML]{ffeede}$1072.4\pm37.1$   \\   & \cellcolor[HTML]{fff2e5}8 leaves (45 params) & \cellcolor[HTML]{fff3e8}8 leaves (69 params) & \cellcolor[HTML]{fff0e3}16 leaves (93 params) & \cellcolor[HTML]{fff3e8}16 leaves (93 params) & \cellcolor[HTML]{fff4ea}16 leaves (141 params) & \cellcolor[HTML]{fff2e5}16 leaves (93 params)    \\  \midrule  \multirow{2}{*}{ICCT-2-feature} & \cellcolor[HTML]{fff4ea}$1000.0\pm0.0$ & \cellcolor[HTML]{ffeddc}$258.4\pm7.0$ &\cellcolor[HTML]{fff0e3}$458.5\pm6.3$ & \cellcolor[HTML]{fff0e3}\textbf{121.9 $\pm$ 0.5} & \cellcolor[HTML]{fff4ea}1280.4$\pm$7.3 & \cellcolor[HTML]{fff0e3} \textbf{1088.6 $\pm$ 21.6}  \\   & \cellcolor[HTML]{fff3e8}4 leaves (29 params) & \cellcolor[HTML]{fff0e3}8 leaves  (101 params) & \cellcolor[HTML]{ffeede} 16 leaves (125 params) & \cellcolor[HTML]{fff2e5}16 leaves (125 params) & \cellcolor[HTML]{fff3e8}16 leaves (205 params) &  \cellcolor[HTML]{ffefe1}16 leaves (125 params) \\  \midrule  \multirow{2}{*}{ICCT-3-feature} & \cellcolor[HTML]{fff4ea}\textbf{1000.0 $\pm$ 0.0} & \cellcolor[HTML]{ffefe1}$275.8\pm1.5$ &\cellcolor[HTML]{ffefe1}$448.8\pm3.0$ &\cellcolor[HTML]{ffeddc}$120.8\pm0.5$ &\cellcolor[HTML]{fff4ea}\textbf{1280.8 $\pm$ 7.7} &\cellcolor[HTML]{ffead7}$1048.7\pm46.7$   \\   & \cellcolor[HTML]{fff4ea}2 leaves (17 params) & \cellcolor[HTML]{ffefe1}8 leaves (133 params) & \cellcolor[HTML]{ffebda}16 leaves (157 params) & \cellcolor[HTML]{ffefe1}16 leaves (157 params) & \cellcolor[HTML]{fff0e3}16 leaves (269 params) &  \cellcolor[HTML]{ffeede}16 leaves (157 params)  \\   \midrule  \multirow{2}{*}{ICCT-L1-sparse} & \cellcolor[HTML]{fff4ea}$1000.0\pm0.0$ & \cellcolor[HTML]{ffeede}$265.2\pm4.3$ &\cellcolor[HTML]{fff2e5}$465.5\pm4.3$ & \cellcolor[HTML]{fff0e3}$121.5\pm0.3$ & \cellcolor[HTML]{fff3e8}$1275.3\pm6.7$ & \cellcolor[HTML]{ffe3c9} $993.2\pm14.6$  \\   & \cellcolor[HTML]{fff3e8}4 leaves (29 params) & \cellcolor[HTML]{ffeddc}8 leaves (165 params) & \cellcolor[HTML]{ffe3c9}16 leaves (253 params) & \cellcolor[HTML]{ffd5ad}16 leaves (765 params) & \cellcolor[HTML]{ffc288}16 leaves (2189 params) &  \cellcolor[HTML]{ffd8b4}16 leaves (509 params)  \\  \midrule \multirow{2}{*}{ICCT-complete} & \cellcolor[HTML]{fff4ea}$1000.0\pm0.0$ & \cellcolor[HTML]{fff4ea}\textbf{300.5 $\pm$ 1.2} & \cellcolor[HTML]{fff3e8}\textbf{476.6 $\pm$ 3.1} & \cellcolor[HTML]{ffeddc}$120.7\pm0.5$ & \cellcolor[HTML]{ffeede}$1248.6\pm3.6$ & \cellcolor[HTML]{ffe3c9} $994.1\pm29.1$  \\   & \cellcolor[HTML]{fff4ea}2 leaves (13 params) & \cellcolor[HTML]{ffeddc}8 leaves (165 params) & \cellcolor[HTML]{ffe3c9}16 leaves (253 params) & \cellcolor[HTML]{ffd5ad}16 leaves (765 params) & \cellcolor[HTML]{ffc288}16 leaves (2189 params) &  \cellcolor[HTML]{ffd8b4}16 leaves (509 params)  \\  \midrule  \multirow{2}{*}{CDDT-controllers Crisp} & \cellcolor[HTML]{ff8710}$84.0\pm10.4$ & \cellcolor[HTML]{ffb065}$-126.6\pm53.5$ & \cellcolor[HTML]{ff850c}$-39826.4\pm21230.0$ & \cellcolor[HTML]{ffbd7e}$97.9\pm12.0$ & \cellcolor[HTML]{ff9228}$639.62\pm160.4$ &  \cellcolor[HTML]{ff7f00}$245.5\pm48.5$  \\   & \cellcolor[HTML]{fff4ea}2 leaves (13 params) & \cellcolor[HTML]{ffeddc}8 leaves (165 params) & \cellcolor[HTML]{ffe3c9}16 leaves (253 params) & \cellcolor[HTML]{ffd5ad}16 leaves (765 params) & \cellcolor[HTML]{ffc288}16 leaves (2189 params) & \cellcolor[HTML]{ffd8b4} 16 leaves (509 params)   \\  \specialrule{.2em}{.1em}{.1em} \multirow{2}{*}{MLP-Lower} & \cellcolor[HTML]{fff4ea}$1000.0\pm0.0$ & \cellcolor[HTML]{ffe8d3}$231.6\pm49.8$  & \cellcolor[HTML]{fff3e8}$474.7\pm5.8$ & \cellcolor[HTML]{fff0e3}\textbf{121.8 $\pm$ 0.6} & \cellcolor[HTML]{ff932a}$646.4\pm151.2$ & \cellcolor[HTML]{ffd1a6}$868.4\pm100.9$  \\   & \cellcolor[HTML]{ffefe1}79 params & \cellcolor[HTML]{fff0e3}110 params & \cellcolor[HTML]{ffeede} 127 params & \cellcolor[HTML]{fff0e3}151 params & \cellcolor[HTML]{fff2e5}221 params & \cellcolor[HTML]{fff2e5} 103 params  \\  \midrule \multirow{2}{*}{MLP-Upper} & \cellcolor[HTML]{fff4ea}$1000.0\pm0.0$ & \cellcolor[HTML]{fff2e5}$288.7\pm2.8$ & \cellcolor[HTML]{fff2e5}$467.9\pm8.5$ & \cellcolor[HTML]{fff0e3}$121.8\pm0.3$ & \cellcolor[HTML]{ffeddc}$1239.5\pm4.2$ &  \cellcolor[HTML]{ffeede}$1077.7\pm31.1$  \\   & \cellcolor[HTML]{ffebda} 121 params & \cellcolor[HTML]{ffe9d5}222 params & \cellcolor[HTML]{ffd7b2}407 params & \cellcolor[HTML]{ffd7b2}709 params & \cellcolor[HTML]{ffa854}3266 params &  \cellcolor[HTML]{ffb977}1021 params  \\  \midrule  \multirow{2}{*}{MLP-Max} & \cellcolor[HTML]{fff4ea}$1000.0\pm0.0$ & \cellcolor[HTML]{fff3e8} \textbf{298.5 $\pm$ 0.7} & \cellcolor[HTML]{fff4ea}\textbf{478.2 $\pm$ 6.7} & \cellcolor[HTML]{fff0e3}$121.7\pm0.4$ & \cellcolor[HTML]{ffca98}$1011.9\pm141.3$ & \cellcolor[HTML]{fff2e5}\textbf{1104.3 $\pm$ 9.4}   \\   & \cellcolor[HTML]{ff7f00} 67329 params & \cellcolor[HTML]{ff7f00}68610 params & \cellcolor[HTML]{ff7f00}69377 params & \cellcolor[HTML]{ff7f00} 77569 params & \cellcolor[HTML]{ff7f00}83458 params & \cellcolor[HTML]{ff7f00}73473 params  \\  \midrule  \multirow{2}{*}{CDDT} & \cellcolor[HTML]{fff4ea}\textbf{1000.0 $\pm$ 0.0} & \cellcolor[HTML]{ffe8d3}$226.4\pm44.5$ & \cellcolor[HTML]{fff2e5}$464.7\pm5.4$ & \cellcolor[HTML]{ffeddc}$120.9\pm0.5$ & \cellcolor[HTML]{ffeede}\textbf{1248.0 $\pm$ 6.4} &  \cellcolor[HTML]{ffe8d3}$1033.2\pm24.1$  \\   & \cellcolor[HTML]{fff4ea}2 leaves (8 params) & \cellcolor[HTML]{fff2e5} 8 leaves (86 params) & \cellcolor[HTML]{ffe5ce}16 leaves (226 params) & \cellcolor[HTML]{ffd8b4}16 leaves (706 params) & \cellcolor[HTML]{ffdec0}16 leaves (1036 params) &  \cellcolor[HTML]{ffdbb9}16 leaves (466 params)  \\  \midrule \multirow{2}{*}{CDDT-controllers} & \cellcolor[HTML]{fff4ea}$1000.0\pm0.0$ & \cellcolor[HTML]{fff2e5}$289.0\pm2.4$  & \cellcolor[HTML]{fff3e8}$469.7\pm11.1$ & \cellcolor[HTML]{ffeddc}$120.1\pm0.3$ & \cellcolor[HTML]{ffeede}$1243.8\pm3.6$ &  \cellcolor[HTML]{ffe5ce}$1010.9\pm25.7$  \\   & \cellcolor[HTML]{fff4ea}2 leaves (16 params) & \cellcolor[HTML]{ffe9d5} 8 leaves (214 params) & \cellcolor[HTML]{ffd6af}16 leaves (418 params) & \cellcolor[HTML]{ffb977}16 leaves (1410 params) & \cellcolor[HTML]{ffc48c}16 leaves (2092 params) & \cellcolor[HTML]{ffbf83}16 leaves (914 params)  \\   \specialrule{.2em}{.1em}{.1em} \end{tabular}

}
\caption{\textcolor{black}{In this table, we display the results of our evaluation. 
For each evaluation, we report the mean ($\pm$ standard error) and the complexity of the model required to generate such a result. Our table is broken into three segments, the first containing equally interpretable approaches that utilize static distributions at their leaves. The second segment contains interpretable approaches that maintain linear controllers at their leaves. The ordering of methods denotes the relative interpretability. The third segments displays black-box approaches. We bold the highest-performing method in each segment, and break ties in performance by model complexity. \textcolor{black}{We color table elements in association with the number of parameters and performance. 
Reddish colors relate to a larger number of policy parameters and lower average reward.}}}
\label{tab:my-table}
\end{sidewaystable*}

\section{Environments}
\label{sec:environments}
Here, we provide short descriptions across six domains used in our extensive evaluation. \textcolor{black}{We start with two common continuous control problems, Inverted Pendulum, and Lunar Lander, provided by OpenAI Gym \citep{brockman2016openai}. We then test across four autonomous driving scenarios: Lane-Keeping provided by \citet{highway-env} and Single-Lane Ring Network, Multi-Lane Ring Network, and Figure-8 Network all provided by the Flow deep reinforcement learning framework for mixed autonomy traffic scenarios \citep{Wu2017FlowAA}.} 

\noindent \textit{Inverted Pendulum:} In Inverted Pendulum \citep{todorov2012mujoco}, a control policy must apply throttle \textcolor{black}{(ranging from +3 to move left to -3 to move right) to balance a pole. The observation includes the cart position, velocity, pole angle, and pole angular velocity.} 

\begin{itemize}[leftmargin=*]
\item \textit{Lunar Lander:} In Lunar Lander \citep{parberry2017introduction,brockman2016openai}, a policy must throttle the main engine and side engine thrusters for a lander to land on a specified landing pad.
The observation is 8-dimensional, including the lander's current position, linear velocity, tilt, angular velocity, and information about ground contact. The continuous action space is two-dimensional, with the first dimension controlling the main engine thruster and the second controlling the side engine thrusters.

\item \textit{Lane-Keeping \citep{highway-env}:} A control policy must control a vehicle's steering angle to stay within a curving lane. 
\textcolor{black}{The observation is 12-dimensional, which consists of lane information and the vehicle's lateral position, heading, lateral speed, yaw rate, and linear, lateral, and angular velocity. 
The action is the steering angle to control the vehicle.}

\item \textit{Flow Single-Lane Ring Network \citep{Wu2017FlowAA}:} A control policy must apply acceleration commands to a vehicle agent to stabilize traffic flow consisting of 21 other human-driven (synthetic) vehicles. 
\textcolor{black}{The observation includes the world position and velocity of all vehicles.}

\item \textit{Flow Multi-Lane Ring Network \citep{Wu2017FlowAA}:} A control policy must apply acceleration and lane-changing commands to an ego vehicle to stabilize the flow of noisy traffic flow across multiple lanes. \textcolor{black}{The observation includes the world position and velocity of all vehicles.}

\item \textit{Flow Figure-8 Network \citep{Wu2017FlowAA}:} A control policy must apply acceleration to a vehicle to stabilize the flow in a Figure-8 network (contains a section where the vehicles must cross paths at the center of the 8), requiring the policy to adapt its control input to create a stable flow through this section. The observation is the world position and velocity of all vehicles.
\end{itemize}

\section{Results}
\label{sec:results}
In this section, we present the set of baselines we test our model, ICCTs, against. Then, we report the results of our approach versus these baselines across the six continuous control domains, as shown in Table~\ref{tab:my-table}.
All presented results are across five random seeds, and all differentiable frameworks are trained via SAC \citep{haarnoja2018soft}. Each tree-based framework is trained while maximizing performance and minimizing the complexity required to represent such a policy, thereby emphasizing interpretability. 
We release our codebase at \href{https://github.com/CORE-Robotics-Lab/ICCT}{\textcolor{blue}{https://github.com/CORE-Robotics-Lab/ICCT}}.
\subsection{Baselines}
\label{subsec:baselines}

We provide a list of baselines alongside abbreviations used for reference and brief definitions below. We compare against interpretable models, black-box models, and models that can be converted post-hoc into an interpretable form. 
For each method, we also include details regarding the number of active parameters utilized in each model (a surrogate measure for model complexity). We briefly list the following notations for an easier understanding of the computation of the number of parameters for each model discussed below. The number of leaf nodes is $N_l$ (the number of decision nodes is $N_l-1$), the dimension of the observation space is $m$, the number of active features within the leaf controllers is $e$, and the dimension of the action space is $d_a$. The calculated number of parameters is denoted as $N_p$ for each model. Our approach, ICCT-$e$-feature, has a number of parameters of $N_p=3(N_l-1)+(2e+1)d_a N_l=(2ed_a+d_a+3)N_l-3$. 

\begin{itemize}[leftmargin=*]
    \item Continuous DDTs (CDDT): We translate the framework of \citet{Silva2021EncodingHD} to function with continuous action-spaces by modifying the leaf nodes to represent static probability distributions. Here, $N_p=(m+2)(N_l-1)+d_a N_l=(d_a+m+2)N_l-m-2$. When converted into an interpretable form post-hoc, this approach is reported as CDDT-crisp which has a number of parameters, $N_p=3(N_l-1)+d_a N_l=(3+d_a)N_l-3$. 
    \item Continuous DDTs with controllers (CDDT-controllers): We modify CDDT leaf nodes to utilize linear controllers rather than static distributions. Here, $N_p=(m+2)(N_l-1)+(m+1)d_a N_l=(md_a+d_a+m+2)N_l-m-2$. When converted into an interpretable form post-hoc, this approach is reported as CDDT-controllers Crisp that has $N_p=3(N_l-1)+(m+1)d_a N_l=(md_a+d_a+3)N_l-3$. 
    \item ICCTs with static leaf distributions (ICCT-static): We modify the leaf architecture of our ICCTs to utilize static distributions for each leaf (i.e., set $e=0$). Comparing ICCT and ICCT-static displays the effectiveness of the addition of sparse linear sub-controllers. Here, the number of parameters can be computed as $N_p=3(N_l-1)+d_a N_l=(3+d_a)N_l-3$. 
    \item ICCT with complete linear sub-controllers (ICCT-complete): We allow the leaf controllers to maintain weights over all features \textcolor{black}{(no sparsity enforced, i.e., $e=m$}). Comparing ICCT-complete and CDDT-controllers displays the effectiveness of the proposed differentiable crispification procedure. Here, the number of parameters can be computed as $N_p=3(N_l-1)+(m+1)d_a N_l=(md_a+d_a+3)N_l-3$. 
    \item ICCT with L1-regularized controllers (ICCT-L1-sparse): We achieve sparsity via L1-regularization applied to ICCT-complete rather than enforcing sparsity directly via the Enforce$\_$Controller$\_$Sparsity procedure. While this baseline produces sparse sub-controllers, there are drawbacks limiting its interpretability. L1-regularization enforces weights to be near zero rather than exactly zero. These small weights must be represented within leaf nodes, and thus, the interpretability of the resulting model is limited. Here, the number of parameters can be computed as $N_p=3(N_l-1)+(m+1)d_a N_l=(md_a+d_a+3)N_l-3$. 
    \item Multi-layer Perceptron (MLP): We maintain three variants of an MLP. The first (MLP-Max) contains a very large number of parameters, typically utilized in continuous control domains. The second (MLP-Upper) maintains approximately the same number of parameters of our ICCTs with sparse leaf controllers during training, including all inactive parameters. The last (MLP-Lower) maintains approximately the same number of \emph{active} parameters as our ICCTs with sparse leaf controllers. The number of parameters of MLP depends on the size of the network, including all weights and bias parameters. 
    \item Decision Tree (DT): We train a DT via CART \citep{Breiman1983ClassificationAR} on state-action pairs generated from MLP-Max. This baseline represents policy distillation from a high-performance black-box policy to an interpretable model. Here, the number of parameters can be computed as $N_p=2(N_l-1)+d_a N_l=(2+d_a)N_l-2$. 
    \item DT w$\backslash$ DAgger: We utilize the DAgger imitation learning algorithm \citep{Ross2011ARO} to train a DT to mimic the MLP-Max policy. The number of parameters can be computed as in the DT baseline above.
\end{itemize}
\subsection{Discussion}
We present the results of our trained policies in Table \ref{tab:my-table}. We provide the performance of each method alongside the associated complexity of each benchmark in Table \ref{tab:my-table} across three sections, with the top section representing interpretable approaches that maintain static distributions at their leaves, the middle section containing interpretable approaches that maintain linear controllers at their leaves, and the bottom section containing black-box methods. 

\textbf{Static Leaf Distributions (Top):} The frameworks of DT, DT w$\backslash$ DAgger, CDDT-Crisp, and ICCT-static maintain similar representations and are equal in terms of interpretability given that the approaches have the same depth. We see that across three of the six control domains, ICCT-static is able to widely outperform both the DT and CDDT-Crisp models. In the remaining three domains, ICCT-static outperforms CDDT-Crisp by a large margin and achieves competitive performance compared to DTs, even without access to a superior expert policy.


\textbf{Linear Controller Leaf (Middle):} Here, we rank frameworks \textcolor{black}{(top-down)} by their relative interpretability. As the sparsity of the sub-controller decreases, the interpretability diminishes. We see that most approaches are able to achieve the maximum performance in the simple domain of Inverted Pendulum. However, CDDT-controllers-crisp encounters an inconsistency issue from the crispification procedure of \citep{Silva2021EncodingHD,Paleja2020InterpretableAP} and achieves very low performance. 
In regards to interpretability-performance tradeoff,
in Inverted Pendulum, we see that as sparsity increases within the sub-controller, a lower-depth ICCT can be used to achieve a equally high-performing policy. We note that across all domains, we do not find such a linear relationship. \textcolor{black}{We provide additional results within Section \ref{sec:interp-performance} that provide deeper insight into the interpretability-performance tradeoff.}

\textbf{Black-Box Approaches (Bottom):} MLP-based approaches and fuzzy DDTs are not interpretable. While the associated approaches perform well across many of the six domains, the lack of interpretability limits the utility of such frameworks in real-world applications \textcolor{black}{such as autonomous driving. We see that in half the domains, highly-parameterized architectures with over 65,000 parameters are required to learn effective policies.}

\noindent\textbf{Comparison Across All Approaches:} We see that across all continuous control domains, CDDT-Crisp and CDDT-controllers Crisp typically are the lowest-performing models. This displays the drawbacks of the crispification procedure of \citep{Silva2021EncodingHD,Paleja2020InterpretableAP} and the resultant performance inconsistency. Comparing our ICCTs to black-box models, we see that in all domains, we parity or outperform deep highly-parameterized models in performance while reducing the number of parameters required by orders of magnitude. In the difficult Multi-Lane Ring scenario, we see that we can outperform MLPs by 33$\%$ on average while achieving a $300$x-$600$x reduction in the number of policy parameters required.

Overall, we find strong evidence for our Interpretable Continuous Control Trees, displaying and validating the ability to at least parity black-box approaches while maintaining high interpretability. Our novel architecture and training procedure provide a strong step towards providing solutions for two grand challenges in interpretableML: (1) Optimizing sparse logical models such as DTs and (10) Interpretable RL.


\section{Qualitative Exposition of ICCT Interpretability}
\label{sec:qualitative}
\begin{figure}[t]
    \centering
    \includegraphics[width=0.75\textwidth]{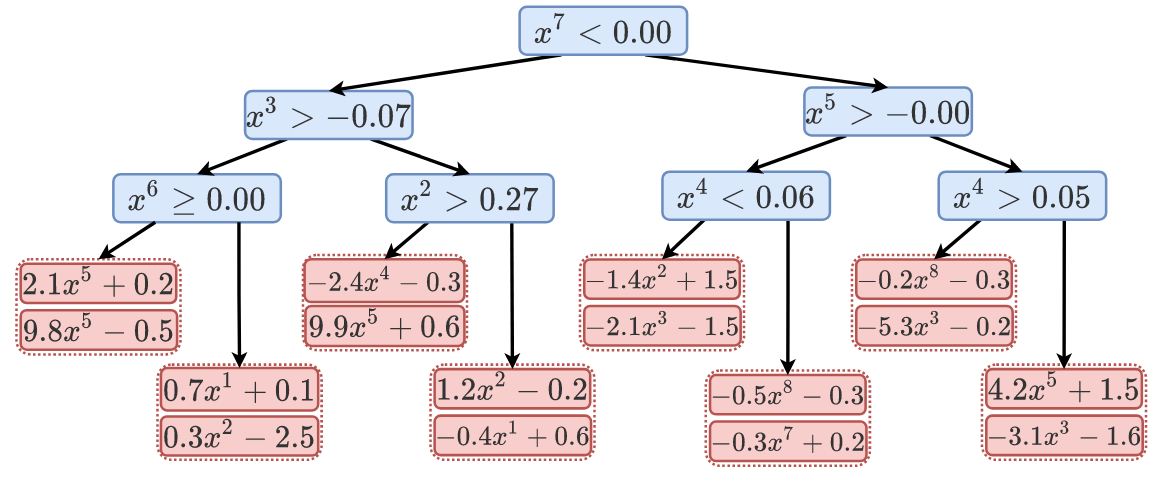}
    \caption{A Learned ICCT in Lunar Lander}
    \label{fig:icct_ll}
\end{figure}
Here, we provide a display of the utility and interpretability of a learned ICCT model. In Figure \ref{fig:icct_ll}, we present our learned ICCT model in Lunar Lander, rounding each element to two decimal places for brevity. The displayed figure is an ICCT-1-feature model (i.e., only one active feature within the sparse sub-controller). The 8-dimensional input in Lunar Lander is composed of position ($x^1$,$x^2$), velocity ($x^3$,$x^4$), angle ($x^5$), angular velocity ($x^6$), left ($x^7$) and right ($x^8$) lander leg-to-ground contact. The action space is two-dimensional: the first (dictated by the top of each pair of the red-colored leaves) controls the main engine thrust, and the second (bottom) controls the net thrust for the side-facing engines. The tree can be interpreted as follows: taking the leftmost path as an example, if the left leg is not touching the ground, the horizontal velocity is greater than -0.07 m/s, and the angular velocity is greater than 0.00 rad/s, then the main engine action is $2.1 * \text{(the lander angle)} + 0.2$, and the side engine action is $9.8 * \text{(the lander angle)} - 0.5$. Such a tree has several use cases: 1) An engineer/developer may pick certain edge cases and verify the behavior of the lander. Performing robustness verification on our ICCTs can be done in linear time (see Section \ref{sec:verification}), while DNN verification is NP-complete \citep{Katz2017ReluplexAE}. 2) An engineer can evaluate the decision-making in the tree and detect anomalies.
Furthermore, there are hands-on use-cases of such a model, such as threshold editing (directly modifying nodes to increase affordances), etc.

\begin{figure}[t]
\centering
    \begin{subfigure}[b]{0.48\textwidth}
    \includegraphics[width=0.94\textwidth, height = 5.2cm]{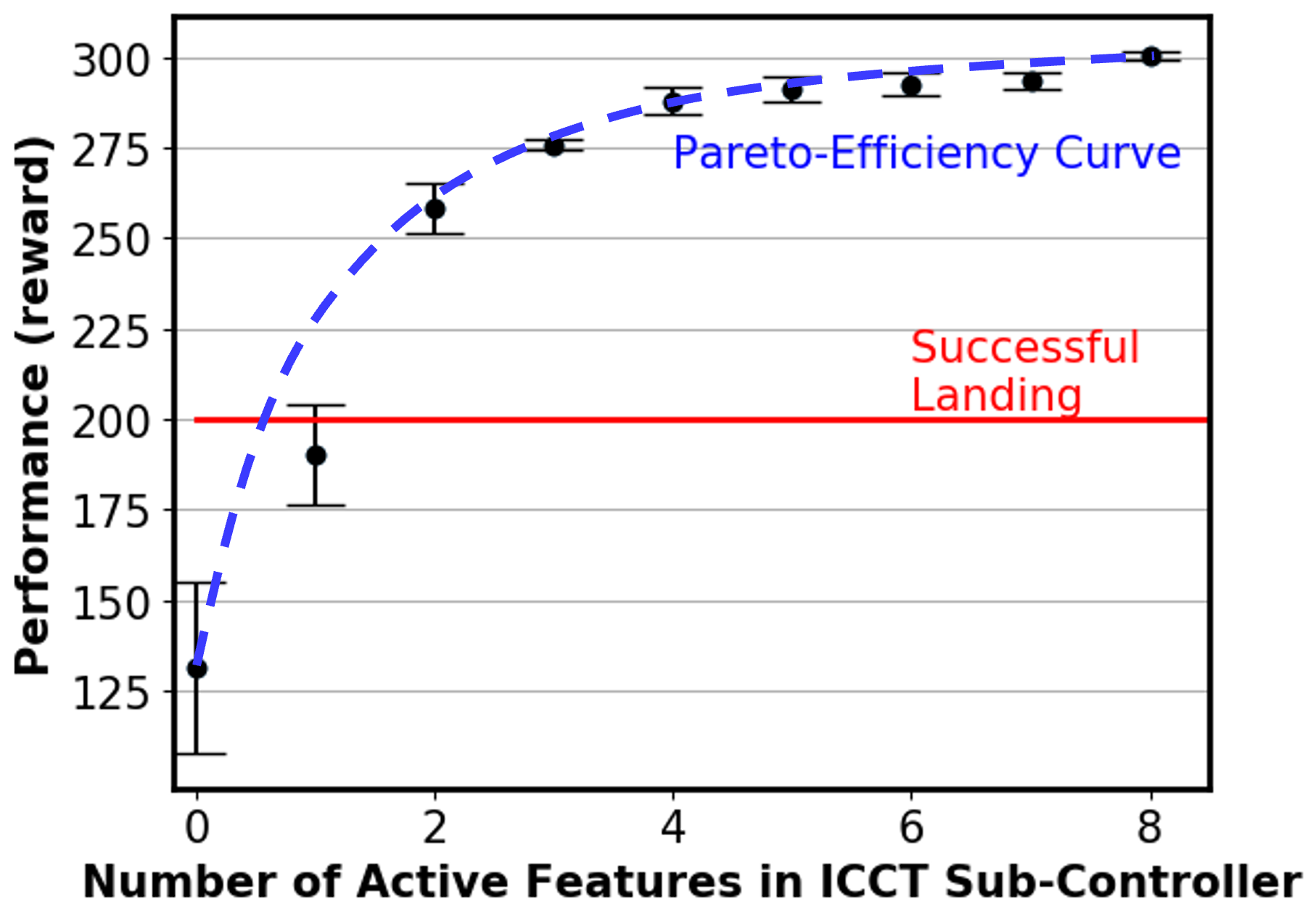}
    \caption{Performance vs. Number of Controller Features}
    \label{fig:pareto_sparse}
    \end{subfigure}
    ~~
    \begin{subfigure}[b]{0.48\textwidth}
    \includegraphics[width=0.89\textwidth,height = 5.2cm]{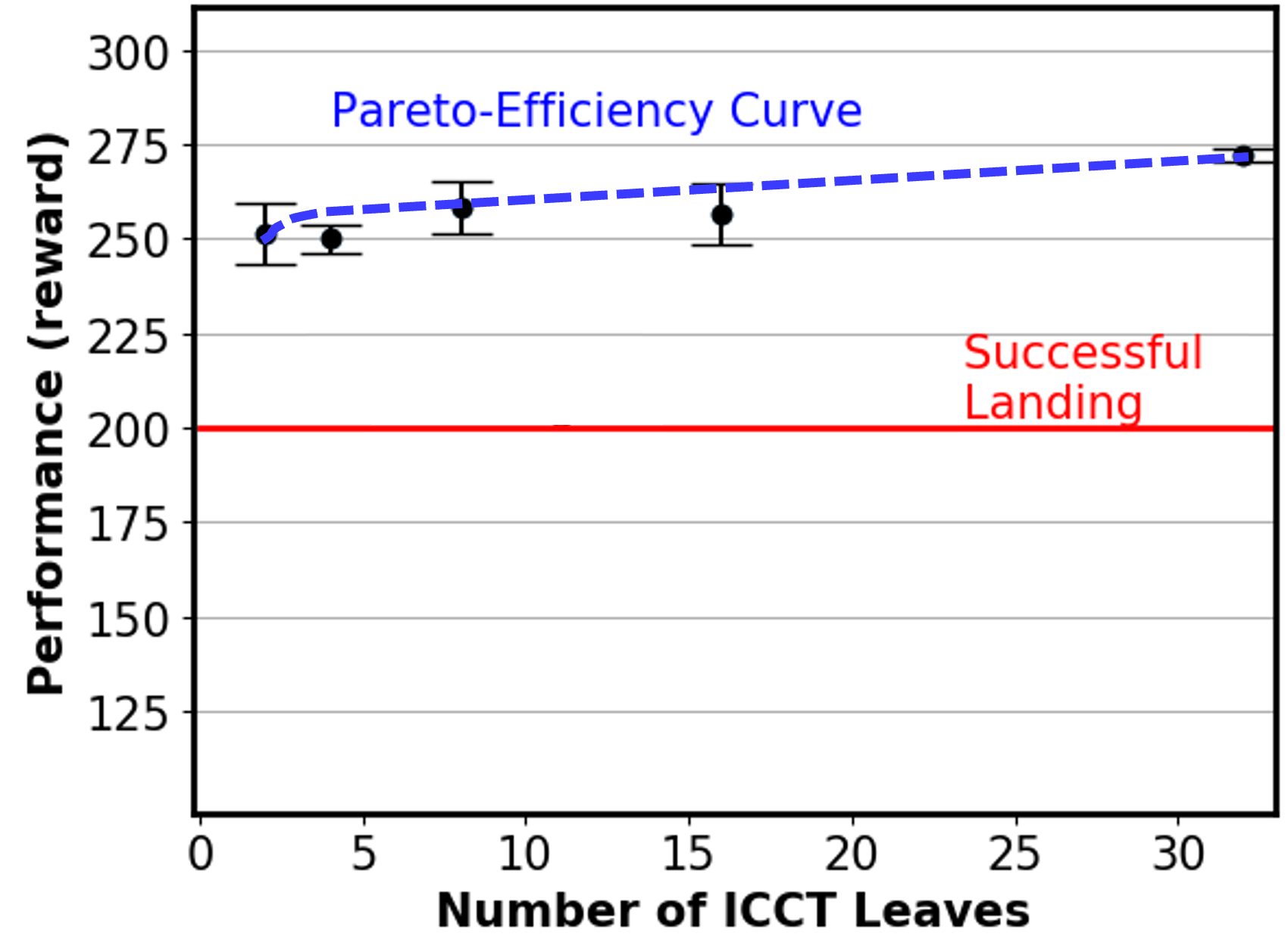}
    \caption{Performance vs. Number of ICCT Leaves}
    \label{fig:pareto_depth}
    \end{subfigure}
    
\caption{In this figure, we display the interpretability-performance tradeoff of our ICCTs with respect to the number of active features within our linear sub-controllers (Figure \ref{fig:pareto_sparse}) and the number of tree leaves (Figure \ref{fig:pareto_depth}) in Lunar Lander. Within each figure, we display the approximate Pareto-Efficiency Curve and denote the reward required for a successful lunar landing as defined by \citep{brockman2016openai}.}
\end{figure}

\section{Ablation: Interpretability-Performance Tradeoff}
\label{sec:interp-performance}
\textcolor{black}{Here, we provide an ablation study over how ICCT performance changes with respect to the number of active features within our linear sub-controllers and depth of the learned policies. \citet{lakkaraju} states that decision trees are interpretable because of their simplicity and that there is a cognitive limit on how complex a model can be while still being understandable. Accordingly, for our ICCTs to maximize interpretability, we emphasize the sparsity of our sub-controllers and attempt to minimize the depth of our ICCTs. Here, we present a deeper analysis by displaying the performance of our ICCTs while varying the number of active features, $e$, from ICCT-static to ICCT-complete (Figure \ref{fig:pareto_sparse}), and varying the number of leaves maintained within the ICCT from $N_l=2$ to $N_l=32$. We conduct our ablation study within Lunar Lander.}

\textcolor{black}{In Figure \ref{fig:pareto_sparse}, we show how the performance of our ICCTs changes as a function of active features in the Sub-Controller. Here, we fix the number of ICCT leaves to 8. We see that as the number of active features increase, the performance also increases. However, there is a tradeoff in interpretability. As greater than 200 reward is considered successful in this domain, a domain expert may determine a point on the Pareto-Efficiency curve that maximizes the interpretability-performance tradeoff.  In Figure \ref{fig:pareto_depth}, we show how the performance of our ICCTs changes as a function of tree depth while fixing the number of active features in the ICCT sub-controller to two. We see a similar, albeit weaker, relationship between performance and interpretability. As model complexity increases, there is a slight gain in performance and a large decrease in interpretability. The Pareto-Efficiency curve provides insight into the interpretability-performance tradeoff for ICCT tree depth.}

\section{Ablation: Differentiable Argument Max and Gumbel-Softmax}
\label{sec:ablation_gumbel}

\begin{figure*}[t]
\centering
\includegraphics[width=\textwidth]{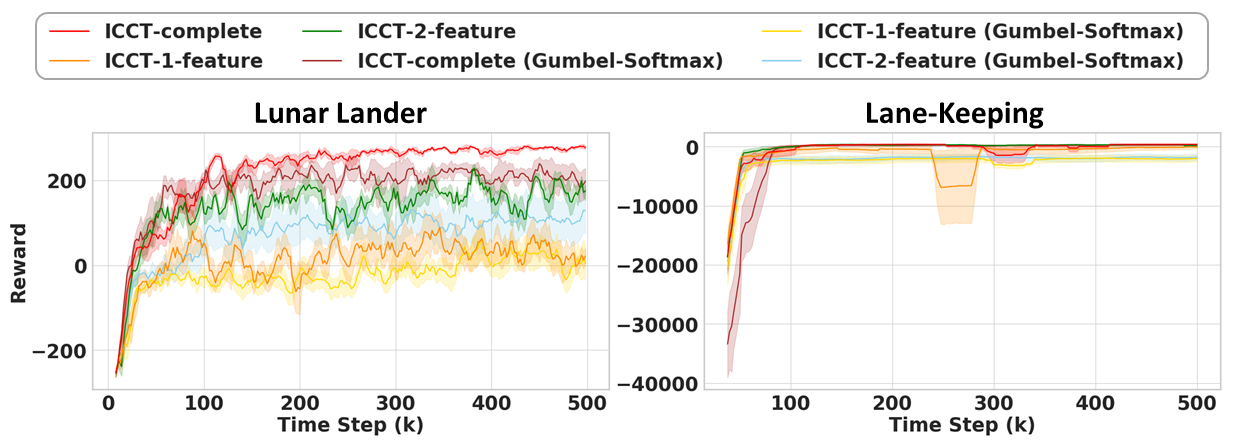}
\caption{This figure displays the average running rollout rewards of six methods for the ablation study during training. The results are averaged over 5 seeds, and the shadow region represents the standard error.}
\label{fig:ablation_gumbel}
\end{figure*}

\begin{table}[t]
\centering 
\begin{tabular}{c c c} 
\hline
Method &  Lunar Lander & Lane-Keeping \\ [0.5ex] 
\hline %
ICCT-complete & $300.5\pm1.2$ & $476.6\pm3.1$ \\
ICCT-complete (Gumbel-Softmax) & $276.7\pm7.0$ & $412.6\pm31.3$ \\
ICCT-complete (Gumbel-Softmax, Crisp) & $239.0\pm18.9$ & $309.1\pm94.6$ \\
\hline
ICCT-1-feature & $190.1\pm13.7$ & $437.6\pm7.0$ \\
ICCT-1-feature (Gumbel-Softmax) & $113.2\pm43.1$ & $-853.4\pm333.2$  \\
ICCT-1-feature (Gumbel-Softmax, Crisp) & $-20.1\pm50.0$ & $-658.114\pm345.3$  \\
\hline
ICCT-2-feature & $258.4\pm7.0$ & $458.5\pm6.3$  \\
ICCT-2-feature (Gumbel-Softmax) & $161.7\pm54.8$ & $-560.6\pm251.6$  \\
ICCT-2-feature (Gumbel-Softmax, Crisp) & $62.3\pm82.2$ & $-945.0\pm331.0$  \\
\hline 
\end{tabular}
\caption{This table shows a performance comparison between ICCTs utilizing our proposed differentiable argument max function (\textsc{diff\_argmax}$(\cdot)$ in Algorithm \ref{alg:argmax}), and a variant of ICCTs utilizing the Gumbel-Softmax function (fuzzy and crisp). Across each approach, we present our findings across Lunar Lander and Lane-Keeping and include ICCTs with fully parameterized sub-controllers (ICCT-complete) and sparse sub-controllers. } 
\label{table:ablation_gumbel} 
\end{table}
\normalsize

In this section, we provide an ablation study on the differentiable operator used in ICCTs to perform decision node crispification, perform decision outcome crispification, and enforce sub-controller sparsity. Here, we substitute the Softmax function with a Gumbel-Softmax \citep{Jang2017CategoricalRW} function, a widely-used differentiable approximate sampling mechanism for categorical variables. To allow ICCTs to utilize the Gumbel-Softmax function, as opposed to \textsc{diff\_argmax}$(\cdot)$, we modify the original Softmax function, $f$, introduced by Equation \ref{eq:softmax}, to $f'$ defined in Equation \ref{eq:gumbel_softmax}. Here, $\vec{w_i}$ is an $m$-dimensional vector, $[w_i^1, \cdots, w_i^m]^T$, and $\{g_i^j\}_{j=1}^m$ are i.i.d samples from a $\text{Gumbel}(0, 1)$ distribution \citep{Jang2017CategoricalRW}.
\begin{equation}
\label{eq:gumbel_softmax}
    f'(\vec{w_i})_k = \frac{\exp{\big(\frac{w_i^k + g_i^k}{\tau}\big)}}{\sum_j^m \exp{\big(\frac{w_i^j + g_i^j}{\tau}\big)}} 
\end{equation}

In Table \ref{table:ablation_gumbel}, we compare the performance of ICCT-complete, ICCT-1-feature, and ICCT-2-feature to their variants using Gumbel-Softmax in Lunar Lander and Lane-Keeping. All the methods and their corresponding variants are trained using the same hyperparameters.
From the results shown in Figure \ref{fig:ablation_gumbel} and Table \ref{table:ablation_gumbel},  we find that the addition of Gumbel noise reduces performance by a wide margin. Furthermore, comparing crisp ICCTs utilizing Gumbel-Softmax to ICCTs utilizing Gumbel-Softmax, we see that due to the sampling procedure within the Gumbel-Softmax, an inconsistency issue arises between fuzzy and crisp performance. Such results support our design choice of the differentiable argument max function over the Gumbel-Softmax.

\section{Physical Robot Demonstration}
\begin{figure*}[t]
    \centering
    \includegraphics[width=\textwidth]{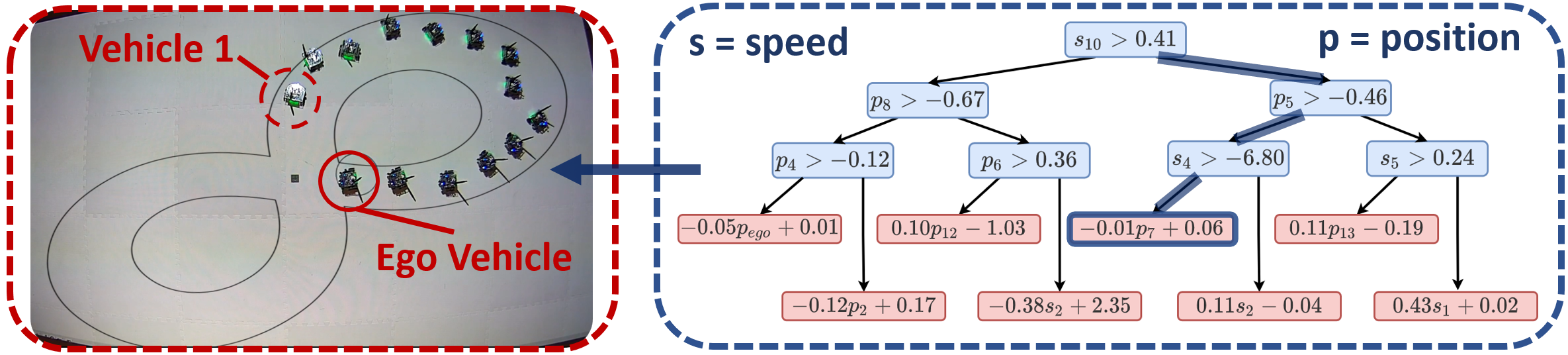}
    \caption{In this figure, we display our ICCTs controlling a vehicle in a 14-car physical robot demonstration within a Figure-8 traffic scenario. Active nodes and edges are highlighted by the right online visualization, where $s_i$ represents the speed of vehicle $i$, and $p_i$ represents the position of vehicle $i$. We include a full video, including an enlarged display of our ICCT at \href{https://sites.google.com/view/icctree}{\textcolor{blue}{https://sites.google.com/view/icctree}}.}
    \label{fig:robotarium}
\end{figure*}
\label{sec:robot}
\textcolor{black}{Here, we demonstrate our algorithm with physical robots in a 14-car figure-8 driving scenario and provide an online, easy-to-inspect visualization of our ICCTs, which controls the ego vehicle. We utilize the Robotarium, a remotely accessible swarm robotic research platform \citep{Wilson2020TheRG}, to demonstrate the learned ICCT policy. The demonstration displays the feasibility of ego vehicle behavior produced by our ICCT policy and provides an online visualization of our ICCTs. A frame taken from the demonstrated behavior is displayed in Figure \ref{fig:robotarium}. We provide a complete video of the demonstrated behavior and the online visualization of the control policy at \href{https://sites.google.com/view/icctree}{\textcolor{blue}{https://sites.google.com/view/icctree}}.}

\section{Case Studies on Complex Driving Domain Grounded in Realistic Lane Geometries}
We extend our analysis of the ICCT with experiments on two realistic driving domains modeled after 1) the I-94 highway in Michigan and 2) the I-280 highway in California, and verify ICCT's ability to drive safely on crowded highways. 

\subsection{The I-94 Domain}
The I-94 domain is built with the SUMO traffic simulator~\citep{lopez2018microscopic} and modeled off the Interstate Highway 94 from Exit 190 to Exit 192. As illustrated in Figure~\ref{fig:i94}, the I-94 domain consists of three ramp entries and three ramp exits. The task of the ego vehicle is to enter from the first ramp entry and exit from the last ramp exit and do so as quickly as possible while being safe. 

\begin{figure*}[t]
    \centering
    \includegraphics[width=\textwidth]{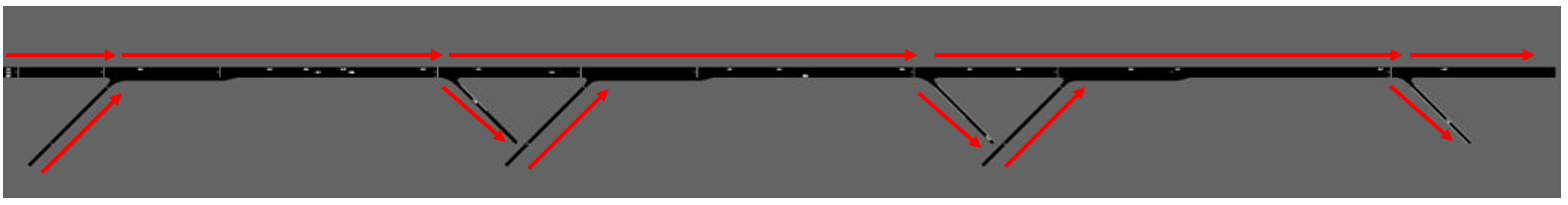}
    \caption{This figure illustrates the I-94 domain. The red arrows denote the traffic flow directions. There are four traffic inflows: one highway inflow (leftmost) and three ramp inflows. There are also four traffic outflows: highway outflow (rightmost) and three ramp outflows. }
    \label{fig:i94}
\end{figure*}
\begin{table*}[t]
\footnotesize
\begin{center}
 \tabcolsep=5pt\relax
\begin{tabular}{lccccccccc}
 \hline
 \multirow{4}{*}{Metrics} & \multicolumn{5}{c}{ICCT} & \multicolumn{3}{c}{MLP} \\
 \cline{2-5}\cline{6-9}
  & 2 leaves & 2 leaves & 4 leaves & 4 leaves & 8 leaves & \multirow{2}{*}{Lower} & \multirow{2}{*}{Upper} & \multirow{2}{*}{Max} \\
  & 1-feature & 2-feature & 1-feature & 2-feature & 1-feature &&&
  \\
  \hline
 Environment Returns & 878.72 & 899.65 & 858.62 & \textbf{910.93} & 871.80 & 907.11 & 901.29 & 897.13 \\
 \# number of parameters & 15 & 23 & 33 & 49 & 69 & 46 & 862 & 71426 \\
 \# collisions & 0 & 0 & 0 & 0 & 0 & 0 & 0 & 1 \\
 \# hard-brakes & 0 & 0 & 0 & 0 & 0 & 0 & 0 & 2 \\
 \# unsafe lane changes & 14 & 7 & 30 & \textbf{0} & 9 & 13 & \textbf{0} & 3 \\
 \hline
\end{tabular}
\caption{This table shows our findings within the I-94 domain. Environment returns are the average of ten evaluation episodes after training has been completed. The remaining metrics are computed through a summation over occurrences of the respective phenomena across the ten evaluation episodes.} 
\label{tab:i94_result}
\end{center}
\end{table*}

The state space of I-94 is 19-dimensional and consists of the ego longitudinal location (i.e., progress of the driving), the ego speed, the ego latitudinal location (i.e., lane information), and the surrounding vehicles' status. We define ``surrounding'' vehicles as the leading cars and following cars on each lane with respect to the ego car, and for each surrounding vehicle, we provide the distance to the ego car as well as its velocity. As there are four lanes on I-94, the surrounding vehicles' information is 16-dimensional. The action space of the RL agent is two-dimensional, which controls the ego vehicle's acceleration and steering angle. 

As we would like to ensure both the performance and safety of the RL agent, we design a 6-component reward function, $R_{\text{I-94}} = \sum_{i=1}^6 R_i$.  
Here, $R_1$ represents the positive reward for ego speed, $R_2$ represents the negative constant time penalty, $R_3$ denotes the negative penalty for too-small headways, $R_4$ provides the negative penalty for wrong routing (i.e., the ego vehicle does not exit properly via the last ramp), $R_5$ encodes the negative penalty for emergency braking behaviors, and $R_6$ is the negative penalty for unsafe lane changes. Intuitively, $R_1$ and $R_2$ motivate the ego car to move faster, $R_4$ helps the ego car to follow the desired route, and $R_3$, $R_5$, and $R_6$ encourage the RL vehicle to be safe by keeping headways and avoid dangerous behaviors. 

\subsection{I-94 Results}
We summarize the results of the I-94 domain in Table~\ref{tab:i94_result}. We compare five sizes of ICCTs and three sizes of MLPs (as introduced in Section \ref{subsec:baselines}) across four metrics. Firstly, we find that the ICCT and MLP can achieve similar performance on the environment returns metric, with ICCT (4 leaves, 2-feature) variant achieving the highest returns. The second and third metrics we consider are the number of collisions and the number of hard brakes in the evaluation of 10 episodes after training the policies. As our driving simulator, SUMO, always prevents an actual collision by applying a high deceleration, we regard a deceleration that is higher than $10m\cdot s^{-2}$ as a collision, and a deceleration that is between $5m\cdot s^{-2}$ and $10m\cdot s^{-2}$ as a hard brake. We observe that most ICCT and MLP models achieve zero collisions and hard brakes, with the exception of MLP-Max. We hypothesize that MLP-Max may overfit to the training data and generate undesired large-model artifacts, showing the brittleness of large MLPs in general. The last metric we consider is the number of unsafe lane changes, which counts the number of instances that the ego car tries to merge onto a lane with too small of a headway or tailway, or tries to change to a non-existing lane (e.g., attempts to change to the left lane on the left-most lane). We observe that the best-performing ICCT model with 4 leaves of 2 features achieves zero unsafe lane changes and higher rewards than the best MLP model. Comparing ICCT-2-feature with 4 leaves and MLP-Upper, although the two models have similar performance on all metrics, the ICCT model maintains $94\%$ fewer parameters compared with the MLP-Upper model. Thus, we conclude that in the case study of the I-94 domain, ICCT models were able to learn to successfully accomplish the task, meeting both performance and safety objectives while being highly parameter-efficient. 

\begin{figure*}[t]
\centering
    \begin{subfigure}[b]{0.80\textwidth}
    \centering
    \includegraphics[width=0.70\textwidth]{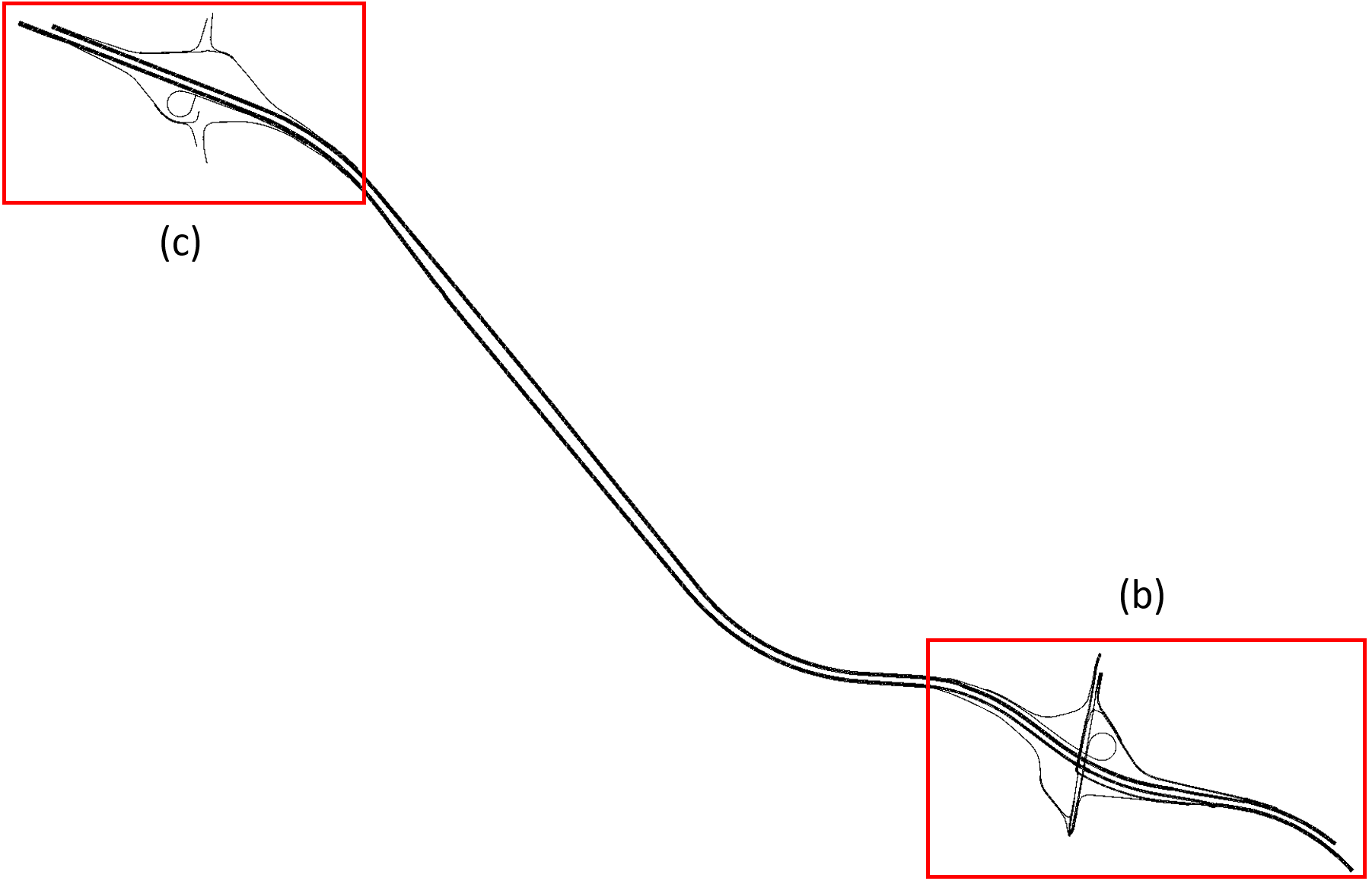}
    \caption{I280-overview}
    \label{fig:I280_overview}
    \end{subfigure}
    ~~
    \begin{subfigure}[b]{0.90\textwidth}
    \centering
    \includegraphics[width=0.90\textwidth]{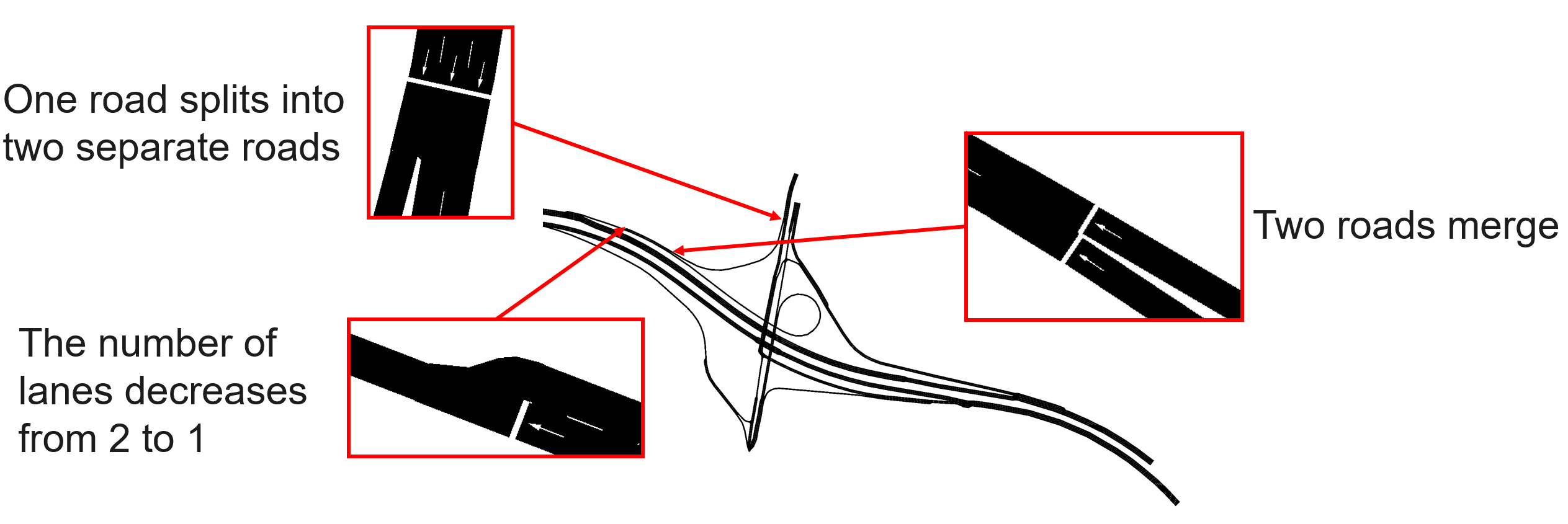}
    \caption{I280-begin}
    \label{fig:I280_begin}
    \end{subfigure}
    ~~
    \begin{subfigure}[b]{0.90\textwidth}
    \centering
    \includegraphics[width=0.90\textwidth]{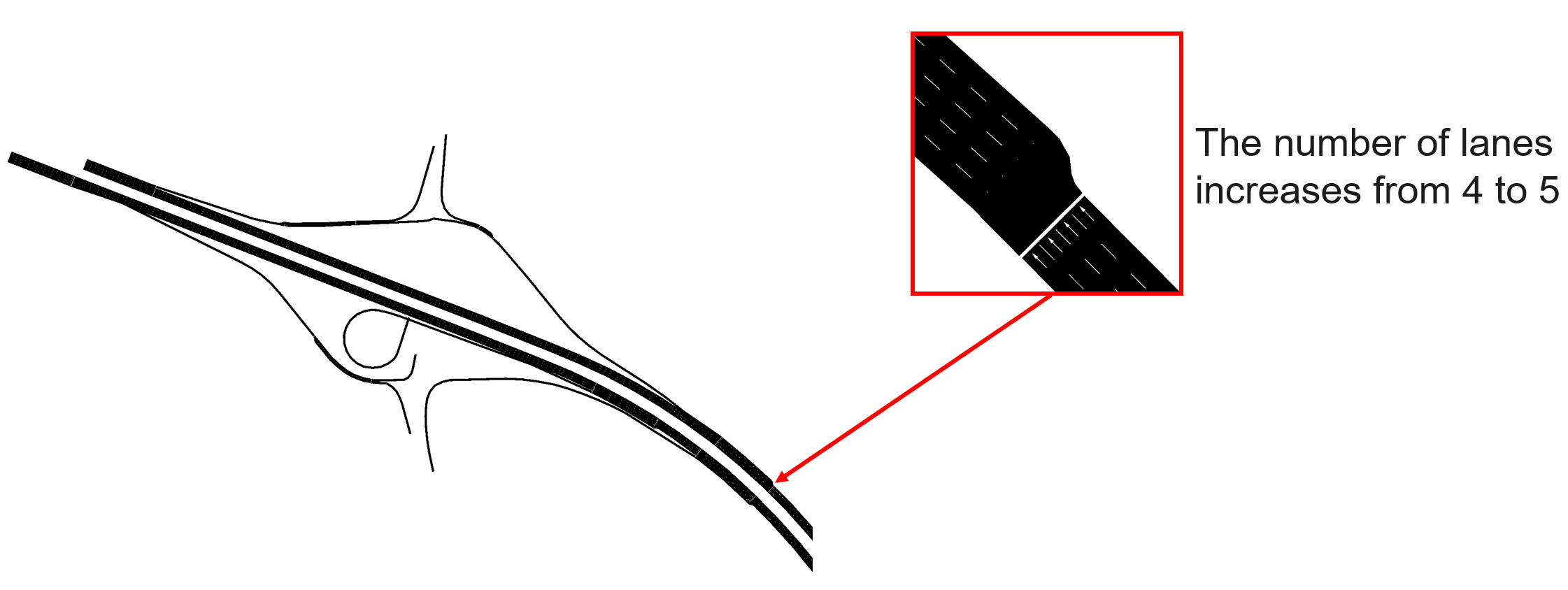}
    \caption{I280-end}
    \label{fig:I280_end}
    \end{subfigure}
\caption{\color{black} Figure \ref{fig:I280_overview} presents the overview of the I-280 domain. The ego vehicle is tasked to join the highway from the ramp and then exit the environment at the end of the highway. The ego vehicle's entrance ramp and exit is zoomed in and presented in Figure \ref{fig:I280_begin} and Figure \ref{fig:I280_end}, respectively. }
\label{fig:I280_env}
\end{figure*}

\subsection{The I-280 Domain}
The I-280 domain is a complex environment modeled off the Interstate Highway 280 around the Palo Alto area in California\footnote{Our I-280 domain can be found here: \href{https://github.com/songanz/flow_evaluation}{\textcolor{blue}{https://github.com/songanz/flow$\_$evaluation}.}}. An overview of the I-280 Domain is shown in Figure~\ref{fig:I280_overview}. The ego vehicle is tasked to join the highway from the ramp and then exit the environment at the end of the highway. As shown in Figure~\ref{fig:I280_begin} and Figure~\ref{fig:I280_end}, the I-280 domain is more challenging than the I-94 domain due to the multiple road geometries encountered across a trajectory, including road splits and the introduction and disappearance of lanes. The state space for I-280 is 19-dimensional and consists of the ego lane index, the ego speed, the speed limit within the current segment, the target lane (i.e., the lanes that lead the ego to the next road segment on its route), ego progress within the current segment, and the surrounding vehicles’ information. In I-280, the surrounding vehicles include the leading cars and
following cars on the lanes adjacent to the ego car (i.e., the left lane, the current lane of the ego car, and the right lane). For each surrounding vehicle, the domain provides its distance to the ego car and velocity information. The action space of the ego agent is two-dimensional, which controls the ego vehicle’s acceleration and steering angle. \emph{There are 119,347 passenger vehicles in total involved within this domain, creating an extremely large-scale simulation with complex, stochastic vehicle interactions.} However, as dense traffic makes the simulation significantly slow, we reduce the traffic density in the original I-280 Domain to involve vehicles that are near to the ego vehicle controlled by our model. 
The modification makes the environment renders faster and boosts the training process. 

As I-280 introduces the notion of the speed limit and complex road routing, we design an 8-component reward function, $R_{\text{I-280}} = \sum_{i=1}^8 R_i$.
Here, $R_1$ represents the positive reward for ego speed, $R_2$ represents the negative penalty if the ego vehicle is not in the target lane, $R_3$ encodes the negative penalty for too-small headways, $R_4$ provides the negative penalty for too-large headways, $R_5$ encodes a negative penalty for emergency braking behaviors, $R_6$ represents a negative penalty for unsafe lane changes, $R_7$ is a progress reward once the ego vehicle finishes $40\%$ and $60\%$ of its route, and $R_8$ is a reward if the ego vehicle arrives at its destination (i.e., accomplishing the task). Intuitively, $R_1$ and $R_4$ motivate the ego vehicle to move faster and follow traffic flow, $R_2$ helps the ego car to follow the correct route, $R_3$, $R_5$, and $R_6$ encourage the RL vehicle to be safe by keeping headways and avoiding dangerous operation, and $R_7$ and $R_8$ encourage the RL vehicle to move further in its trajectory.

\begin{figure*}[t]
    \centering
    \includegraphics[width=0.6\textwidth]{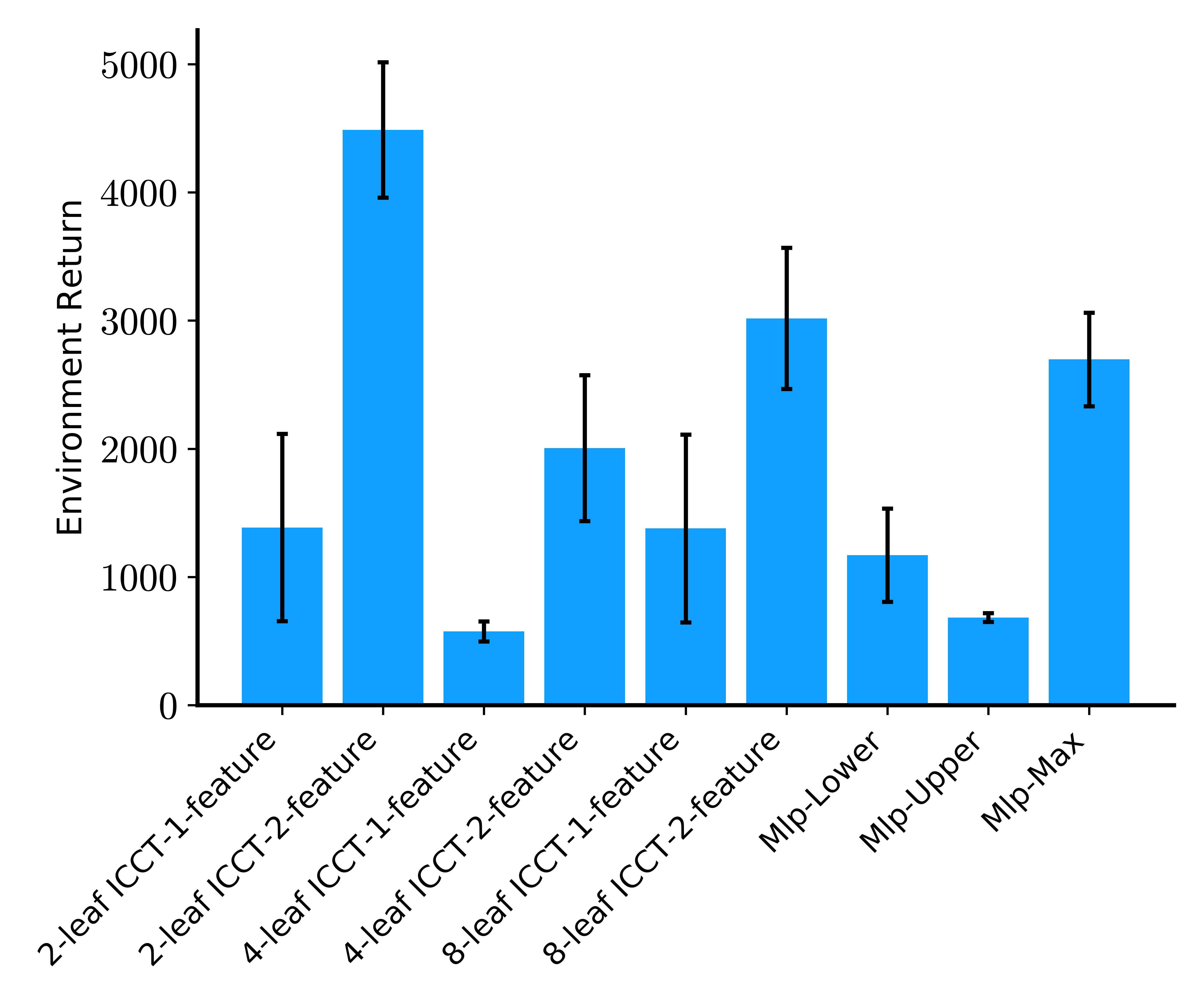}
    \caption{This figure compares the performance of ICCT agents and MLP agents in the I-280 domain. The environment returns for each model are displayed through a mean and standard deviation across ten evaluation episodes.}
    \label{fig:I280_performance}
\end{figure*}

\subsection{I-280 Results}
We have summarized the results of the I-280 case study in Figure~\ref{fig:I280_performance}. We compare six sizes of the ICCT and three sizes of an MLP. The ``environment returns'' metric represents the average episode reward in 10 rollouts. The environment returns for the 2-feature ICCT with 2 leaves and 2-feature ICCT with 8 leaves are significantly higher than the MLP models, demonstrating the ICCT's capability on more complex and realistic domains. We also observe that ICCT-2-Feature performs better than ICCT-1-feature. 

\textit{Summary:} Across both realistic autonomous driving domains, we find that ICCTs can serve as safe continuous control models that follow traffic regulations and maintain high performance with respect to the domain objectives. 

\section{Interpretability User Study}

In the previous sections, we have validated ICCT's efficacy in learning high-performance, safe policies. In this section, we seek to verify ICCT's interpretability with end-users through a human-subject experiment. To get a comprehensive understanding of ICCT's interpretability compared with neural networks, we design a $3\times3\times2\times2$ experiment. We describe the four independent variables as follows:

\textit{Models (3-levels):} 1) \textit{Tree}, 2) \textit{MLP}, and 3) \textit{Paragraph}. We compare the interpretability across the three models. The condition, \textit{Tree}, refers to an ICCT in its original tree form (e.g., Figure \ref{fig:icct_ll}). The condition, \textit{Paragraph}, refers to a text description encompassing a set of rules extracted from an ICCT. We obtain this form by first training an ICCT and then transforming it into a text description after training has been completed. During the transformation, each leaf node corresponds to a sub-paragraph, which is formed by two components, namely the conditions and the consequence. The conditions of a sub-paragraph describe the junction of all decision nodes on the path from the root node to the leaf, and the consequence is the leaf node itself. For example, a four-leaf tree corresponds to four sub-paragraphs. We hypothesize such a \textit{Paragraph} form may positively contribute to the interpretability of the ICCT as it does not assume prior familiarity with tree-based models. The condition, \textit{MLP}, displays a multi-layer perceptron.

\textit{Repeats of evaluation (3-level):} 1st, 2nd, and 3rd. To test the interpretability of the three models, we task participants to calculate the output of the model given the model parameters and its inputs. For each model, we ask participants to make predictions with the same model three times, each with varying inputs. This allows us to attain a lower-variance estimation of both the participant's accuracy and time spent in model computation. This condition also allows us to compare the learning effect of interpreting each of the models (i.e., improvements from repeated evaluations). 

\textit{Contexts (2-levels):} 1) With Context and 2) Without Context. We investigate whether providing context (i.e., a visualization of the traffic scenario) contributes to the user's rating of model interpretability. This condition helps us understand whether each model's interpretability is its inherent property from its model structure or is related to specific example contexts. 

\textit{Domains (2-levels):} 1) Multi-Lane Ring and 2) I-94. We examine whether models' interpretability is impacted by the environment's complexity. 

\begin{figure*}[t]
    \centering
    \includegraphics[width=\textwidth]{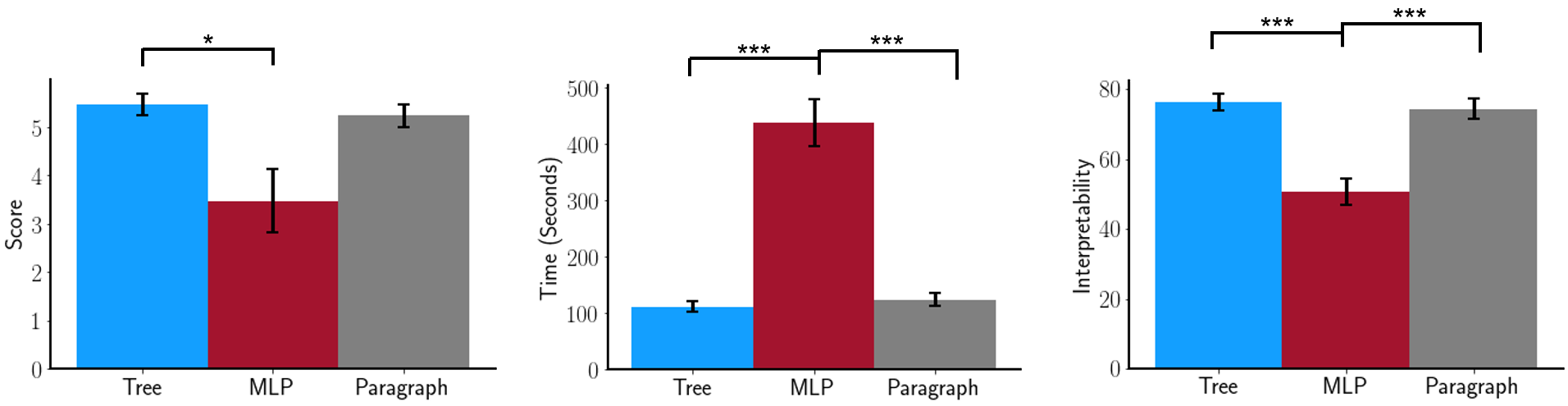}
    \caption{This figure shows the comparisons of accuracy score (left), time spent (middle), and subjective interpretability rated (right) across the three models in the I-94 user study. $^{*}$ denotes a significant difference of $p<.05$. $^{***}$ denotes a significant difference of $p<.001$. }
    \label{fig:h1-3}
\end{figure*}

Out of the four factors, we design \textit{Models} and \textit{Repeats of evaluation} to be within-subject factors as we seek to test for the learning effects and choose \textit{Contexts} and \textit{Domains} to be between-subject factors. We test for the following six hypotheses from the user study. 

\begin{itemize}[leftmargin=*]
    \item H1: \textit{Tree} and \textit{Paragraph} are easier to simulate than neural network, i.e., the simulation of \textit{Tree} and \textit{Paragraph} has higher accuracy than \textit{MLP}. 
    \item H2: \textit{Tree} and \textit{Paragraph} are quicker to validate than MLP. 
    \item H3: \textit{Tree} and \textit{Paragraph} are more interpretable than \textit{MLP}, measured by a 13-item Likert questionnaire introduced by \citep{Paleja2020InterpretableAP}. 
    \item H4: Performance improvement by repeated evaluations of \textit{Tree} is larger than of \textit{MLP}. 
    \item H5: Environment context of the decision-making increases interpretability. 
    \item H6: The advantage of \textit{Tree} and \textit{Paragraph}'s subjective interpretability over \textit{MLP} is domain-independent. 
\end{itemize}

To test for the six hypotheses, we build an online survey with Qualtrics. In each question of the survey, we show a model (\textit{Tree}, \textit{MLP}, or \textit{Paragraph}) and an input feature vector. The subject is asked to make predictions (i.e., compute the output of the model) given the respective input. We collect $N=34$ responses. The average age of participants is $25.15$ with a standard deviation of $4.28$. Out of the $34$ participants, 15 are male, and 19 are female. 

\begin{figure*}[t]
    \centering
    \includegraphics[width=0.8\textwidth]{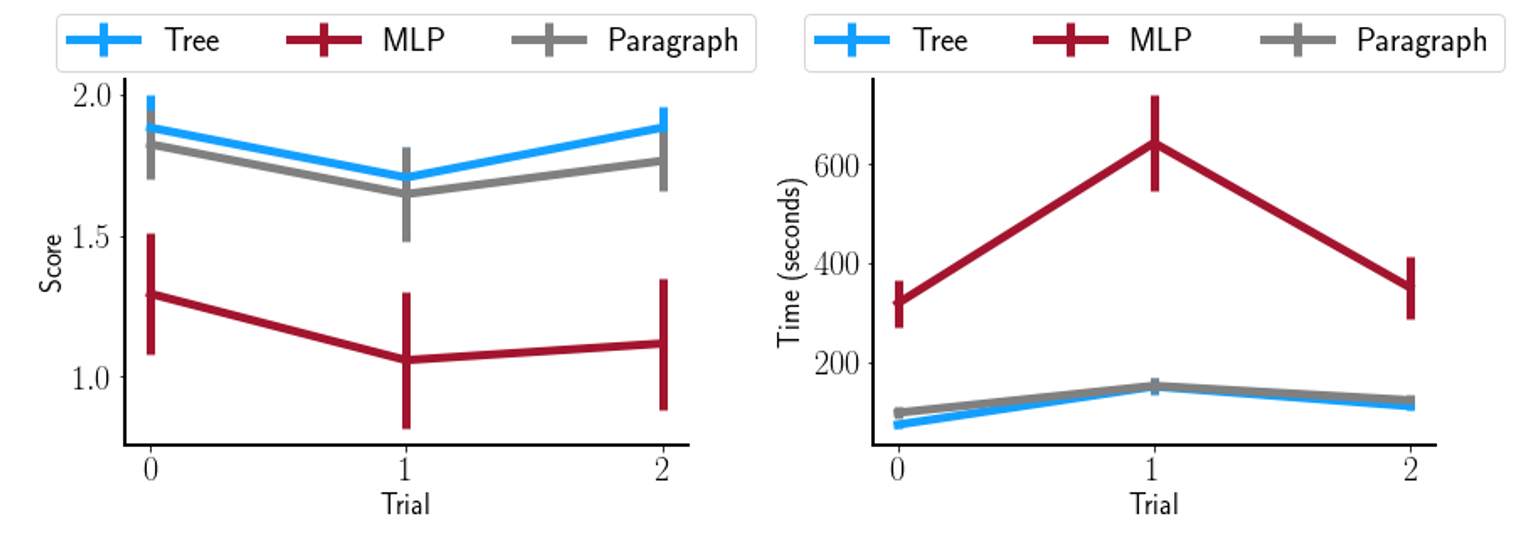}
    \caption{This figure shows the accuracy score (left) and time spent (right) changes in three repeats trials across the three models in the user study. }
    \label{fig:h4}
\end{figure*}

\subsection{User Study Results}

For H1-H5, we illustrate results on the I-94 domain, and for H6, we show the results on both the Multi-lane Ring and I-94 domains to verify the conclusions hold for both environments. For all statistical tests, the assumptions for the ANOVA test are not satisfied. Thus, we instead perform a non-parametric Friedman test followed by a posthoc Nemenyi–Damico–Wolfe (Nemenyi) test~\citep{damico1987extended}.

\begin{figure*}[t]
    \centering
    \includegraphics[width=0.6\textwidth]{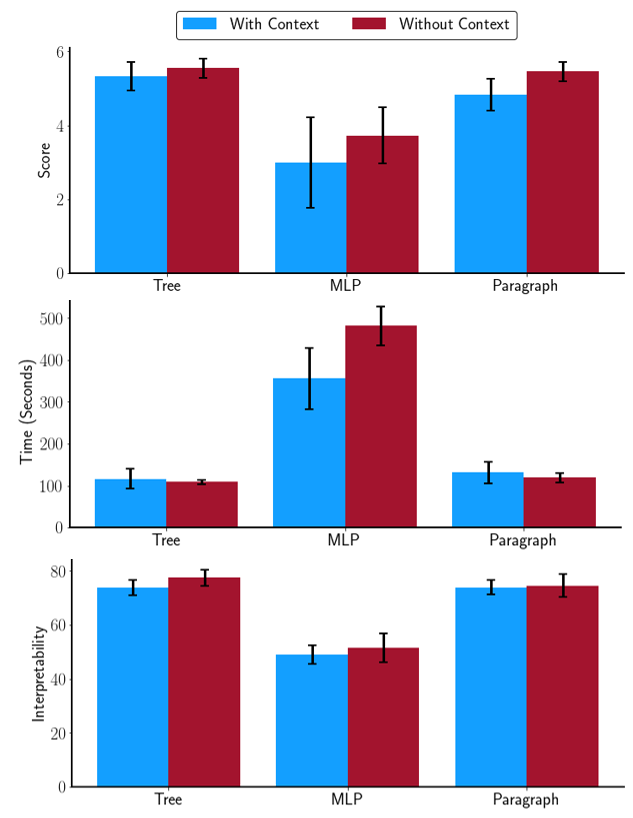}
    \caption{This figure shows the comparisons of accuracy score (left), time spent (middle), and interpretability rated (right) with or without context across the three models in the user study. }
    \label{fig:h5}
\end{figure*}

\begin{figure*}[t]
    \centering
    \includegraphics[width=0.8\textwidth]{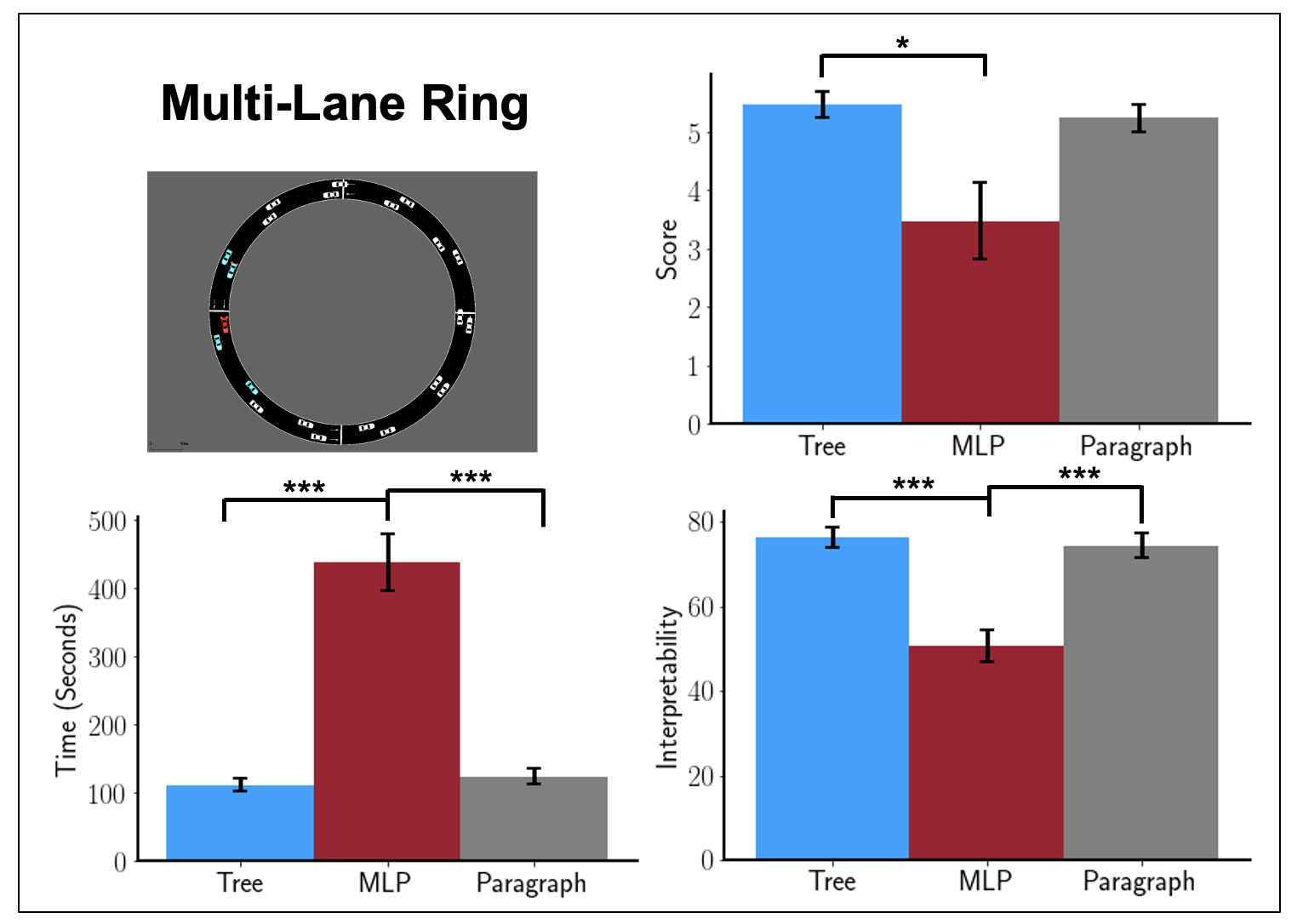}
        \includegraphics[width=0.795\textwidth]{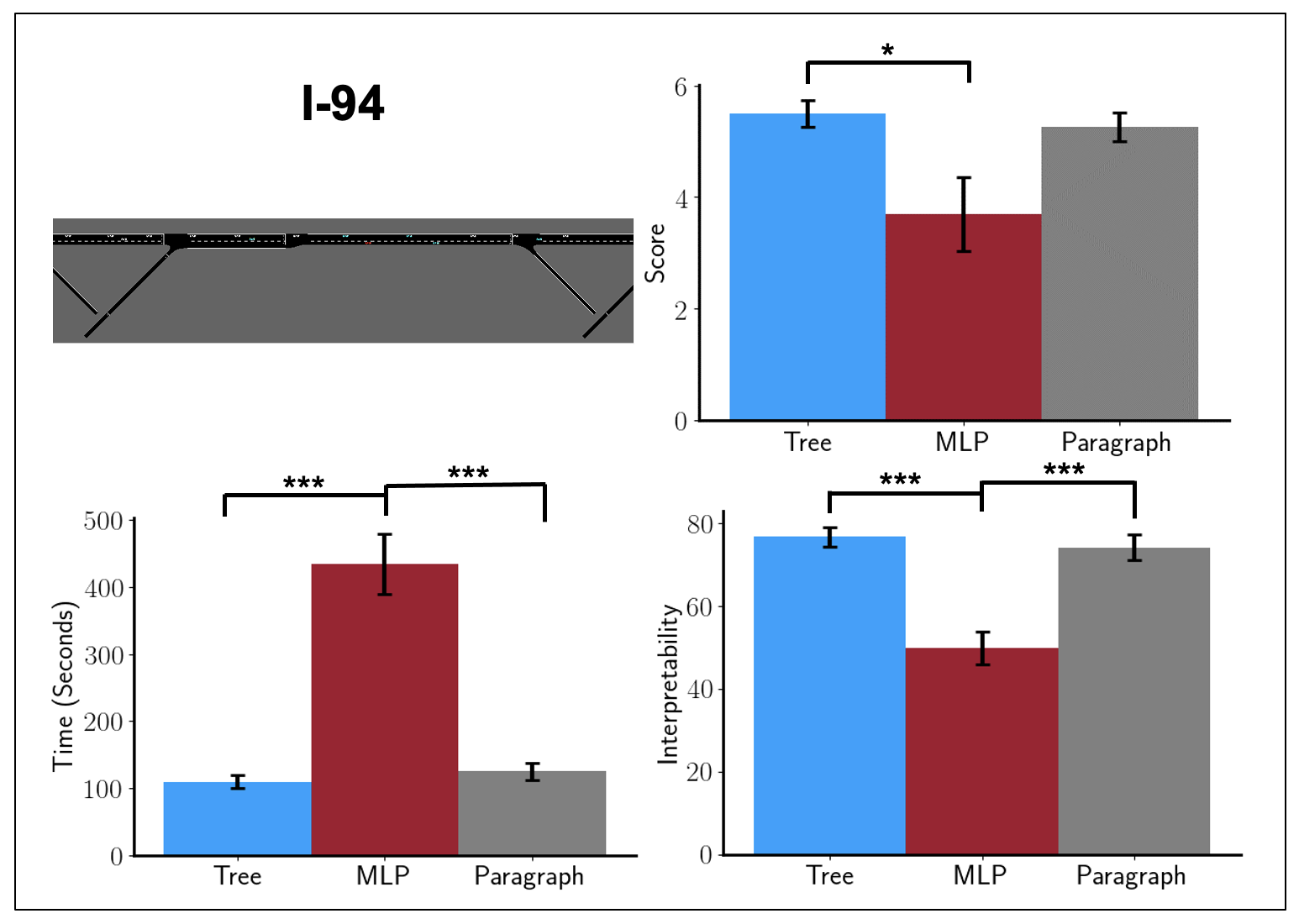}
    \caption{This figure shows the comparison of results between the Multi-Lane Ring domain and the I-94 domain across the three models in the user study. $^{*}$ denotes a significant difference of $p<.05$. $^{***}$ denotes a significant difference of $p<.001$. }
    \label{fig:h6}
\end{figure*}

\textit{H1-H3:}
We summarize the results for H1-H3 in Figure~\ref{fig:h1-3}. As we test the user's simulation of outputs on each model three times and the action output is two-dimensional (acceleration and lane-changing), each user is validated across six action outputs for each model. We denote the number of accurate answers out of the six as the ``score'' in Figure~\ref{fig:h1-3} left. We observe that participants are able to simulate the outputs of the \textit{ICCT} more accurately than an \textit{MLP} ($p<.05$), supporting H1. Furthermore, the time spent on the \textit{ICCT} and \textit{Paragraph} to evaluate the output is significantly less than the time spent on MLP ($p<.001$), shown in Figure~\ref{fig:h1-3} (middle). \textit{Tree} and \textit{Paragraph} are also rated by users to have significantly higher interpretability ($p<.001$) (Figure~\ref{fig:h1-3} right). As such, the results from the user study support H1-H3, showing \textit{Tree} and \textit{Paragraph} are easier to simulate, quicker to validate, and more interpretable than neural networks. We find there is no significant difference across all three metrics between the tree form and the paragraph form of the ICCT. 

\textit{H4:}
We show the results to test for H4 in Figure~\ref{fig:h4}. We hypothesize that with practice, the ability of the users to evaluate the models may increase. However, Figure~\ref{fig:h4} shows that the accuracy on the first trial is the highest, and the time spent on the first trial is the lowest. Instead of the learning effect, the finding may be explained by a fatiguing effect, which causes the participants to have lower performance in later trials. Another hypothesis is that the first trial questions are relatively easy as they correspond to the early stages of the execution, where the environment has fewer cars. The changes in accuracy score and time across three trials do not have a significant interaction effect with models, and therefore H4 is not supported. However, across repeated iterations, we observe that the score and the prediction time of the ICCT tree form and paragraph form are close while being much easier to simulate than an MLP. 

\textit{H5:}
We illustrate the result for H5 in Figure~\ref{fig:h5}. We observe that generally, for all three models, the condition with context results in slightly lower accuracy scores and interpretability scores. For \textit{Tree} and \textit{Paragraph}, the time spent with context is slightly higher than the condition without context. However, the time spent on MLP without context is higher than with context. One possible reason could be that tree depictions and paragraph descriptions are interpretable and quick to evaluate, and therefore context does not provide more benefit but introduces some workload overhead. For MLP, the context helps the user to understand the situation and therefore makes the evaluation faster. Overall, H5 is not supported as context does not provide a significant boost to subjective interpretability. 

\textit{H6:}
The comparison of results between the two domains, Multi-Lane Ring and I-94, can be viewed in Figure~\ref{fig:h6}. We observe that the results for H1-H3 are similar for both domains and therefore, H6 is supported by displaying the advantage of different representations of the ICCT's (both the \textit{Tree} and \textit{Paragraph} form) interpretability over an \textit{MLP} is regardless of the two domains ($p<.05$ on accuracy score and $p<.001$ on both time and subjective interpretability).

\section{Conclusion}
In this work, we present a novel tree-based model for continuous control, applicable to a wide variety of domains including robotic manipulation, autonomous driving, etc. Our Interpretable Continuous Control Trees (ICCTs) have competitive performance to that of deep neural networks across several continuous control domains, including six difficult autonomous driving scenarios and two driving domains grounded in realistic lane geometries, while maintaining high interpretability. The maintenance of high performance within an interpretable and verifiable reinforcement learning architecture provides a paradigm that would be beneficial for the safe real-world deployment of autonomous systems. 

\color{black}
\section{Limitations and Future Work:} 

Our framework has several limitations. Continuous control outputs (e.g., predicting a steering angle) may not be interpretable to end-users and may require post-processing to enhance a user's understanding. Also, the relationship between controller sparsity, tree depth, and interpretability is not clear, making controller sparsity and tree depth difficult-to-define hyperparameters. We also note that in more challenging environments, larger ICCTs may be required to increase their representative power. In these instances, although the size of Interpretable Continuous Control Trees (ICCTs) may pose challenges for end-users in terms of interpretation, it is important to note that they can still be verified by experts and interpreted within specific tree sub-spaces.In future work, we will extend ICCTs to incorporate constraints from end-users, provide safety guarantees on our ICCTs, and reason about how the complexity of an ICCT may change as we move to higher-abstraction state spaces.

\acks{This work was supported by a gift award from the Ford Motor Company, National Science Foundation (NSF) 1757401 (SURE Robotics), NSF CNS 2219755, and a research grant from MIT Lincoln Laboratory. 

DISTRIBUTION STATEMENT A. Approved for public release. Distribution is unlimited. This material is based upon work supported by the Under Secretary of Defense for Research and Engineering under Air Force Contract No. FA8702-15-D-0001. Any opinions, findings, conclusions or recommendations expressed in this material are those of the author(s) and do not necessarily reflect the views of the Under Secretary of Defense for Research and Engineering. \\
Delivered to the U.S. Government with Unlimited Rights, as defined in DFARS Part 252.227-7013 or 7014 (Feb 2014). Notwithstanding any copyright notice, U.S. Government rights in this work are defined by DFARS 252.227-7013 or DFARS 252.227-7014 as detailed above. Use of this work other than as specifically authorized by the U.S. Government may violate any copyrights that exist in this work.}

\color{black}





\bibliography{sample}

\end{document}